\documentclass[12pt,letterpaper]{article}

\usepackage{microtype}
\usepackage{graphicx}
\usepackage{caption}
\usepackage{subcaption}
\usepackage{booktabs} 
\usepackage[margin=3cm]{geometry}
\usepackage{hyperref}

\usepackage{amsmath}
\usepackage{amssymb}
\usepackage{mathtools}
\usepackage{amsthm}

\usepackage[capitalize,noabbrev]{cleveref}

\theoremstyle{plain}
\newtheorem{theorem}{Theorem}[section]
\newtheorem{proposition}[theorem]{Proposition}

\theoremstyle{definition}
\newtheorem{definition}[theorem]{Definition}

\theoremstyle{remark}
\newtheorem{remark}[theorem]{Remark}

\usepackage[textsize=tiny]{todonotes}

\usepackage{amsmath,amsfonts,bm}

\def\eqref#1{equation~\ref{#1}}

\def\1{\bm{1}}

\DeclareMathAlphabet{\mathsfit}{\encodingdefault}{\sfdefault}{m}{sl}
\SetMathAlphabet{\mathsfit}{bold}{\encodingdefault}{\sfdefault}{bx}{n}

\newcommand{\R}{\mathbb{R}}

\usepackage{hyperref}
\usepackage{url}

\usepackage{graphicx} 
\usepackage{wrapfig}
\usepackage{amsmath}
\usepackage[english]{babel}
\usepackage{amsthm}
\usepackage{amssymb}

\usepackage{authblk}

\begin{document}
\author[1]{Rocio Diaz Martin*}
\author[2]{Ivan Medri*}
\author[2]{Yikun Bai*}
\author[2]{Xinran Liu}
\author[2]{\\Kangbai Yan}
\author[3]{Gustavo K. Rohde}
\author[2]{Soheil Kolouri}
\affil[1]{Department of Mathematics, Vanderbilt University}
\affil[2]{Department of Computer Science, Vanderbilt University}
\affil[3]{Department of Biomedical Engineering, Department of Electrical and Computer Engineering, Virginia University}
\affil[*]{Equal contribution}
\affil[1]{rocio.p.diaz.martin@vanderbilt.edu}
\affil[2]{ivan.v.medri@vanderbilt.edu}
\affil[2]{yikun.bai@vanderbilt.edu}
\affil[2]{xinran.liu@vanderbilt.edu}
\affil[2]{kangbai.yan@vanderbilt.edu}
\affil[3]{gustavo@virginia.edu}
\affil[2]{soheil.kolouri@vanderbilt.edu}
\date{}

\title{LCOT: Linear circular optimal transport}

\maketitle

\begin{abstract}
The optimal transport problem for measures supported on non-Euclidean spaces has recently gained ample interest in diverse applications involving representation learning. In this paper, we focus on circular probability measures, i.e., probability measures supported on the unit circle, and introduce a new computationally efficient metric for these measures, denoted as Linear Circular Optimal Transport (LCOT). The proposed metric comes with an explicit linear embedding that allows one to apply Machine Learning (ML) algorithms to the embedded measures and seamlessly modify the underlying metric for the ML algorithm to LCOT. We show that the proposed metric is rooted in the Circular Optimal Transport (COT) and can be considered the linearization of the COT metric with respect to a fixed reference measure. We provide a theoretical analysis of the proposed metric and derive the computational complexities for pairwise comparison of circular probability measures. Lastly, through a set of numerical experiments, we demonstrate the benefits of LCOT in learning representations of circular measures.
\end{abstract}

\section{Introduction}
Optimal transport (OT) \cite{villani2009optimal,peyre2019computational} is a mathematical framework that seeks the most efficient way of transforming one probability measure into another. The OT framework leads to a geometrically intuitive and robust metric on the set of probability measures, referred to as the Wasserstein distance. It has become an increasingly popular tool in machine learning, data analysis, and computer vision \cite{kolouri2017optimal,khamis2023earth}. OT's applications encompass generative modeling \cite{arjovsky2017wasserstein,tolstikhin2017wasserstein,kolouri2018sliced2}, domain adaptation \cite{courty2017joint,damodaran2018deepjdot}, transfer learning \cite{alvarez2020geometric,liu2022wasserstein}, supervised learning \cite{frogner2015learning}, clustering \cite{ho2017multilevel}, image and pointcloud registration \cite{haker2004optimal,bai2022sliced,le2023diffeomorphic}, and even inverse problems \cite{mukherjee2021end}, among others. Recently, there has been an increasing interest in OT for measures supported on manifolds \cite{bonet2022spherical,sarrazin2023linearized}. This surging interest is primarily due to: 1) real-world data is often supported on a low-dimensional manifold embedded in larger-dimensional Euclidean spaces, and 2) many applications inherently involve non-Euclidean geometry, e.g., geophysical data or cortical signals in the brain. 

In this paper, we are interested in efficiently comparing probability measures supported on the unit circle, aka circular probability measures, using the optimal transport framework. Such probability measures, with their densities often represented as circular/rose histograms, are prevalent in many applications, from computer vision and signal processing domains to geology and astronomy. For instance, in classic computer vision, the color content of an image can be accounted for by its hue in the HSV space, leading to one-dimensional circular histograms. Additionally, local image/shape descriptors are often represented via circular histograms, as evidenced in classic computer vision papers like SIFT \cite{lowe2004distinctive} and ShapeContext \cite{belongie2000shape}. In structural geology, the orientation of rock formations, such as bedding planes, fault lines, and joint sets, can be represented via circular histograms \cite{twiss1992structural}. In signal processing, circular histograms are commonly used to represent the phase distribution of periodic signals \cite{levine2002signal}. Additionally, a periodic signal can be normalized and represented as a circular probability density function (PDF).

Notably, a large body of literature exists on circular statistics \cite{jammalamadaka2001topics}. More specific to our work, however, are the seminal works of \cite{delon2010fast} and \cite{rabin2011}, which provide a thorough study of the OT problem and transportation distances on the circle (see also \cite{cabrelli1998linear}). OT on circles has also been recently revisited in various papers \cite{hundrieser2022statistics,bonet2022spherical}, further highlighting the topic's timeliness. Unlike OT on the real line, generally, the OT problem between probability measures defined on the circle does not have a closed-form solution. This stems from the intrinsic metric on the circle and the fact that there are two paths between any pair of points on a circle (i.e., clockwise and counter-clockwise). Interestingly, however, when one of the probability measures is the Lebesgue measure, i.e., the uniform distribution, the 2-Wasserstein distance on the circle has a closed-form solution, which we will discuss in the Background section.


We present the Linear Circular OT (LCOT), a new transport-based distance for circular probability measures. By leveraging the closed-form solution of the circular 2-Wasserstein distance between each distribution and the uniform distribution on the circle, our method sidesteps the need for optimization. Concisely, we determine the Monge maps that push the uniform distribution to each input measure using the closed-form solution, then set the distance between the input measures based on the disparities between their respective Monge maps.  Our approach draws parallels with the Linear Optimal Transport (LOT) framework proposed by \cite{wang2013linear} and can be seen as an extension of the cumulative distribution transform (CDT) presented by \cite{park2018cumulative} 
to circular probability measures (see also, \cite{aldroubi2021signed,aldroubi2021partitioning}). 
The idea of linearized (unbalanced) optimal transport was also studied recently in various works \cite{cai2022linearized,moosmuller2023linear,sarrazin2023linearized,cloninger2023linearized}. From a geometric perspective, we provide explicit logarithmic and exponential maps between the space of probability measures on the unit circle and the tangent space at a reference measure (e.g., the Lebesgue measure) \cite{wang2013linear,cai2022linearized,sarrazin2023linearized}. Then, we define our distance in this tangent space, giving rise to the terminology `Linear' Circular OT. The logarithmic map provides a linear embedding for the LCOT distance, while the exponential map inverts this embedding. We provide a theoretical analysis of the proposed metric, LCOT, and demonstrate its utility in various problems. 
 
\noindent{\textbf{Contributions.}} Our specific contributions in this paper include 1) proposing a computationally efficient metric for circular probability measures, 2) providing a theoretical analysis of the proposed metric, including its computational complexity for pairwise comparison of a set of circular measures, and 3) demonstrating the robustness of the proposed metric in manifold learning, measure interpolation, and clustering/classification of probability measures.

\section{Background}
\label{sec: background}

\subsection{Circle Space}
The unit circle $\mathbb{S}^1$ can be defined as the quotient space
$$\mathbb{S}^1={\mathbb{R}}/{\mathbb{Z}}=\left\{\{x+n: \, n\in \mathbb{Z}\}: \,  x\in[0,1)\right\}.$$
The above definition is equivalent to $[0,1]$ under the identification $0=1$. For the sake of simplicity in this article, we treat them as indistinguishable. Furthermore, we adopt a parametrization of the circle as $[0,1)$, designating the North Pole as $0$ and adopting a clockwise orientation. This will serve as our \textit{canonical} parametrization.


Let $|\cdot|$ denote the absolute value on $\mathbb{R}$. With the aim of avoiding any confusion, when necessary, we will denote it by  $|\cdot|_{\mathbb{R}}$.
Then, a metric on $\mathbb{S}^1$ can be defined as 
 $$|x-y|_{\mathbb{S}^1}:=\min\{|x-y|_{\mathbb{R}},1-|x-y|_{\mathbb{R}}\}, \qquad x,y\in[0,1)$$
or, equivalently,  as
 $$|x-y|_{\mathbb{S}^1}:=\min_{k\in \mathbb{Z}} |x-y+k|_{\mathbb{R}}, $$
where for the second formula $x,y\in\mathbb{R}$ are understood as representatives of two classes of equivalence in  $\mathbb{R}/\mathbb{Z}$, but these two representatives $x,y$ do not need to belong to $[0,1)$. It turns out that such a metric defines a geodesic distance:  It is the smaller of the two arc lengths between the points $x$, $y$ along the circumference (cf. \cite{jammalamadaka2001topics}, where here we parametrize angles between $0$ and $1$ instead of between $0$ and $2\pi$. 
Besides, the circle $\mathbb{S}^1$ can be endowed with a group structure. Indeed, as the quotient space $\mathbb{R}/\mathbb{Z}$ it inherits the addition from $\mathbb{R}$ modulo $\mathbb{Z}$. Equivalently, for any $x,y\in [0,1)$, we can define the operations $+,-$ as  
    \begin{equation}
   (x,y)\mapsto \begin{cases}
        x\pm y, &\text{ if }x\pm y\in[0,1) \\ 
        x\pm y\mp1, & \text{ otherwise} . 
    \end{cases}     
    \end{equation}

\subsection{Distributions on the circle} 

Regarded as a set, $\mathbb{S}^1$ can be identified with $[0,1)$. Thus, signals over $\mathbb{S}^1$ can be interpreted as 1-periodic functions on $\R$. More generally, every measure $\mu\in\mathcal{P}(\mathbb{S}^1)$ can be regarded as a measure on $\mathbb{R}$ by
\begin{equation}\label{eq: extended measure on R}
  {\mu}(A+n):=\mu(A), \qquad \text{ for every } A\subseteq[0,1) \text{ Borel subset, and } \, n\in\mathbb{Z}.  
\end{equation}
Then, its \textit{cumulative distribution function}, denoted by  $F_\mu$, is defined as 
\begin{equation}\label{eq: CDF on [0,1)}
  F_{\mu}(y):=\mu([0,y))=\int_0^y d\mu, \qquad \forall y\in[0,1)  
\end{equation}
and can be extended to a function on $\mathbb{R}$ by 
\begin{equation}\label{eq: extended CDF}
  F_{\mu}(y+n):=F_{\mu}(y)+n,\qquad \forall y\in[0,1), \, n\in\mathbb{Z}. 
\end{equation}
Figure \ref{fig: CDF} shows the concept of $F_\mu$ and its extension to $\mathbb{R}$.

In the rest of this article, we do not distinguish between the definition of measures on $\mathbb{S}^1$ or their periodic extensions into $\R$, as well as between their CDFs or their extended CDFs into $\R$. 

\begin{figure}[ht]
   \begin{center}
       \includegraphics[width=0.5\linewidth]{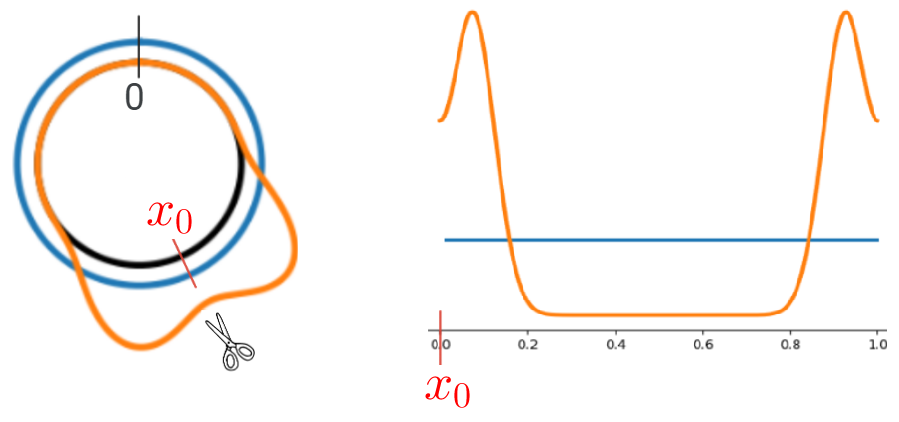}
    \end{center}
    \vspace{-.2in}
    \caption{Visualization of densities (blue and orange) on $\mathbb{S}^1$ and after unrolling them to $[0,1)$ by considering  a cutting point $x_{0}$. The blue density is the uniform distribution on $\mathbb{S}^1$, represented as having height $1$ over the unit circle in black.}
    \label{fig: cut_pdf}
\end{figure}

\begin{definition}[Cumulative distribution function with respect to a reference point] \label{def: CDF cut}
Let $\mu\in\mathcal{P}(\mathbb{S}^1)$, and consider a reference point $x_0\in\mathbb{S}^{1}$. Assume that $\mathbb{S}^1$ is identified as $[0,1)$ according to our canonical parametrization. By abuse of notation, also denote by $x_0$ the point in $[0,1)$ that corresponds to the given reference point when considering the canonical parametrization. We define
\begin{equation*}
    F_{\mu, x_0}(y):=F_\mu(x_0+y)-F_{\mu}(x_0).
\end{equation*}
\end{definition}

The reference point $x_0$ can be considered as the ``origin'' for parametrizing the circle as $[0,1)$ starting from $x_0$. That is, $x_0$ will correspond to $0$, and from there, we move according to the clockwise orientation.
Thus, we can think of $x_0$ in the above definition as a ``cutting point'': A point where we cut $\mathbb{S}^1$ into a line by $x_0$ and so we can unroll PDFs and CDFs over the circle into $\mathbb{R}$. See Figures \ref{fig: cut_pdf} and \ref{fig: CDF}.

Besides, note that $F_{\mu,x_0}(0)=0$ and $F_{\mu,x_0}(1)=1$ by the 1-periodicity of $\mu$. This is to emphasize that in the new system of coordinates, or in the new parametrization of $\mathbb{S}^1$ as $[0,1)$ starting from $x_0$, the new origin $x_0$ plays the role of $0$. 
Finally, notice that if $x_0$ is the North Pole, which corresponds to $0$ in the canonical parametrization of the circle, then $F_{\mu,x_0}=F_\mu$.

\begin{figure}[ht]
    \centering
    \includegraphics[width=1\linewidth]{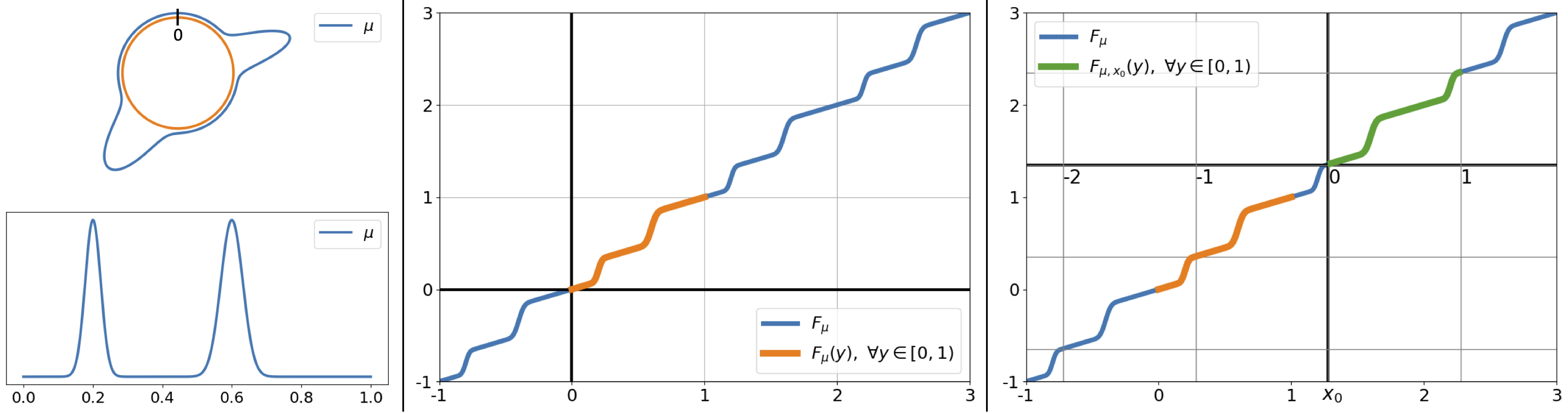}
    \caption{Left: The density of a probability measure, $\mu$. Middle: visualization of the periodic extension to $\mathbb{R}$ of a CDF, $F_\mu$, of measure $\mu$ on   $[0,1)\sim\mathbb{S}^1$.  Right: Visualization of $F_{\mu,x_0}$ given in Definition \ref{def: CDF cut}, where the parameterization of the circle is changed; now, the origin $0$ is the point $x_0$.}
    \label{fig: CDF} 
\end{figure}

\begin{definition}\label{def: quantile}
    The quantile function $F^{-1}_{\mu,x_0}:[0,1]\to [0,1]$ is defined as 
\begin{align}
F^{-1}_{\mu,x_0}(y):=\inf\{x: \, F_{\mu,x_0}(x)>y \}.\nonumber 
\end{align}
\end{definition}

\subsection{Optimal transport on the circle}

\subsubsection{Problem setup}\label{sec: problem set up}
Given $\mu,\nu \in \mathcal{P}(\mathbb{S}^1)$, let $c(x,y):=h(|x-y|_{\mathbb{S}^1})$ be the cost of transporting a unit mass from $x$ to $y$ on the circle,  where $h:\mathbb{R}\to \mathbb{R}_+$ is a convex increasing function. The Circular Optimal Transport cost between $\mu$ and $\nu$ is defined as 
\begin{equation}\label{eq: COT def}
    COT_{h}(\mu,\nu):=\inf_{\gamma\in\Gamma(\mu,\nu)}\int_{\mathbb{S}^1\times\mathbb{S}^1} c(x,y) \, d\gamma(x,y),
\end{equation}       
where $\Gamma(\mu,\nu)$ is the set of all transport plans from $\mu$ to $\nu$, that is, $\gamma\in\Gamma(\mu,\nu)$ is such that $\gamma \in \mathcal{P}(\mathbb{S}^1\times\mathbb{S}^1)$ having first and second marginals $\mu$ and $\nu$, respectively. There always exists a minimizer $\gamma^*$ of \ref{eq: COT def}, and it is called a Kantorovich optimal plan (see, for example, \cite[Th. 1.4]{sant2015}). 

When  $h(x)=|x|^p$, for $1\leq p<\infty$, we denote $COT_{h}(\cdot,\cdot)=COT_p(\cdot,\cdot)$, and $COT_p(\cdot,\cdot)^{1/p}$ defines a distance on $\mathcal{P}(\mathbb{S}^1)$. In general,
\begin{equation}\label{eq: monge formulation}
  COT_{h}(\mu,\nu)\leq \inf_{M:\, M_{\#}\mu=\nu}\int_{\mathbb{S}^1}h(|M(x)-x|_{\mathbb{S}^1}) \, d\mu(x),  
\end{equation}
and a minimizer $M^*:\mathbb{S}^1\to \mathbb{S}^1$ of the right-hand side of \ref{eq: monge formulation}, among all maps $M$ that pushforward $\mu$ to $\nu$ \footnote{The pushforward $M_\#\mu$ is defined by the change of variables $\int\varphi(y) \, d M_\#\mu(y):=\int\varphi(M(x)) \, d\mu(x)$, for every continuous function $\varphi:\mathbb{S}^1\to \mathbb{C}$.}, might not exist. In this work, we will consider the cases where a minimizer $M^*$ does exist, for example, when the reference measure $\mu$ is absolutely continuous with respect to the Lebesgue measure on $\mathbb{S}^1$ (see \cite{mccann2001polar, sant2015}). In these scenarios, such map $M^*$ is called an optimal transportation map or a Monge map. Moreover, as $\mu,\nu\in\mathcal{P}(\mathbb{S}^1)$ can be regarded as measures on $\mathbb{R}$ according to \eqref{eq: extended measure on R}, we can work with transportation maps $M:\mathbb{R}\to \mathbb{R}$ that are $1$-periodic functions satisfying $M_\#\mu=\nu$.

\begin{proposition}\label{prop: equivalent COT}
Two equivalent  formulations of $COT_{h}$ are the following: 
\begin{align}
COT_{h}(\mu,\nu)&=\inf_{x_0\in[0,1)}\int_{0}^1h(|F_{\mu,x_0}^{-1}(x)-F_{\nu,x_0}^{-1}(x)|_{\mathbb{R}})  \, dx \label{eq: cot solution x0}\\ 
&=\inf_{\alpha\in\mathbb{R}}\int_{0}^1h(|F_{\mu}^{-1}(x)-F_{\nu}^{-1}(x-\alpha)|_{\mathbb{R}}) \, dx \label{eq: cot alpha}.
\end{align}    

When there exist minimizers $x_{cut}$ and $\alpha_{\mu,\nu}$ of \eqref{eq: cot solution x0} and \eqref{eq: cot alpha}, respectively, the relation between them is given by
\begin{equation}\label{eq: equation for alpha}
    \alpha_{\mu,\nu} = F_\mu({x_{cut}})- F_\nu(x_{cut}). 
\end{equation}
Moreover, if
$\mu= Unif(\mathbb{S}^1)$ and $h(x)=|x|^2$, 
it can be verified that $\alpha_{\mu,\nu}$ is the antipodal of $\mathbb{E}(\nu)$, i.e.,
\begin{equation}\label{eq: alpha mean}
\alpha_{\mu,\nu}= x_{cut} - F_\nu(x_{cut}) = \mathbb{E}(\nu) -1/2 .
\end{equation} 
\end{proposition}

The proof of \eqref{eq: cot alpha} in Proposition \ref{prop: equivalent COT} is provided in \cite{delon2010fast} for the optimal coupling for any pair of probability measures on $\mathbb{S}^1$. For the particular and enlightening case of discrete probability measures on $\mathbb{S}^1$, we refer the reader to \cite{rabin2011}. In that article, \eqref{eq: cot solution x0} is introduced. Finally, \eqref{eq: alpha mean} is given for example in \cite[Proposition 1]{bonet2022spherical}. 

Proposition \ref{prop: equivalent COT} allows us to see the optimization problem of transporting measures supported on the circle as an optimization problem on the real line by looking for the best “cutting point” so that the circle can be unrolled into the real line by 1-periodicity.

\begin{remark}\label{remark: not unique x_cut but unique alpha}
The minimizer of \eqref{eq: cot alpha} is unique (see \eqref{eq: equation for alpha}), but there can be multiple minimizers for \eqref{eq: cot solution x0} (see Figure \ref{fig: alpha vs x_cut} in Appendix \ref{sec: appendix alpha vs x_cut}).
However, when a minimizer $x_{cut}$ of \eqref{eq: cot solution x0} exists, it will lead to the optimal transportation displacement on a circular domain (see Section \ref{sec: closed formula for COT} below). 
\end{remark}

\subsubsection{A closed-form formula for the optimal circular displacement}\label{sec: closed formula for COT}

Let $x_{cut}$ be a minimizer of \eqref{eq: cot solution x0}, that is, 
\begin{gather}\label{eq: cot as classic ot}
  COT_{h}(\mu,\nu) 
  =\int_0^1 h(|F_{\mu, x_{cut}}^{-1}(x)-F_{\nu,x_{cut}}^{-1}(x)|_{\mathbb{R}}) \, dx.
\end{gather}
From \eqref{eq: cot as classic ot}, one can now emulate the Optimal Transport Theory on the real line (see, for e.g., \cite{sant2015}): The point $x_{cut}$ provides a reference where one can “cut” the circle. Subsequently, computing the optimal transport between $\mu$ and $\nu$ boils down to solving an optimal transport problem between two distributions on the real line. 

We consider the parametrization of $\mathbb{S}^1$ as $[0,1)$ by setting $x_{cut}$ as the origin and moving along the clockwise orientation. Let us use the notation $\widetilde{x}\in [0,1)$ for the points given by such parametrization, and the notation $x\in [0,1)$ for the canonical parametrization. That is, the change of coordinates from the two parametrizations is given by $x = \widetilde{x} + x_{cut}$. Then, if $\mu$ does not give mass to atoms, by \eqref{eq: cot as classic ot} and the classical theory of Optimal Transport on the real line, the optimal transport map (Monge map) that takes a point $\widetilde{x}$ to a point $\widetilde{y}$ is given by
\begin{equation}\label{eq: circ monge map from x_0}
    F_{\nu,x_{cut}}^{-1}\circ F_{\mu,x_{cut}}( \widetilde{x})=\widetilde{y}
\end{equation}
That is, \ref{eq: circ monge map from x_0} defines a circular optimal transportation map from $\mu$ to $\nu$ written in the parametrization that establishes $x_{cut}$ as the ``origin.'' 
If we want to refer everything to the original labeling of the circle, that is, if we want to write \eqref{eq: circ monge map from x_0} with respect to the canonical parametrization, we need to change coordinates
\begin{equation}\label{eq: change or coord}    
\begin{cases}
    \widetilde{x} =& x - x_{cut}  \\
    \widetilde{y} =& y - x_{cut}
\end{cases}.
\end{equation}
Therefore, a closed-form formula for an optimal circular transportation map in the canonical coordinates is given by
\begin{equation}\label{eq: circ monge map}
    M_{\mu}^{\nu}(x):=F_{\nu,x_{cut}}^{-1}\circ F_{\mu,x_{cut}}(x-x_{cut})+x_{cut}=y, \qquad x\in [0,1),
\end{equation}
and the corresponding optimal circular transport displacement that takes $x$ to $y$ is
\begin{equation}\label{eq: circ opt disp}
     M_{\mu}^{\nu}(x)-x=F_{\nu,x_{cut}}^{-1}\circ F_{\mu,x_{cut}}(x-x_{cut})-(x-x_{cut}), \qquad x\in [0,1).
\end{equation}
In summary, we condense the preceding discussion in the following result. The proof is provided in Appendix \ref{sec: Appendix proofs}. While the result builds upon prior work, drawing from \cite{bonet2022spherical,rabin2011,sant2015}, it offers an explicit formula for the optimal Monge map, a detail previously lacking in the literature. 

\begin{theorem}\label{thm: cot embedding} Let $\mu,\nu\in\mathcal{P}(\mathbb{S}^1)$. Assume that $\mu$ is absolutely continuous with respect to the Lebesgue measure on $\mathbb{S}^1$ (that is, it does not give mass to atoms).
\begin{enumerate}
    \item  If  $x_{cut}$ is a minimizer of \eqref{eq: cot solution x0}, then \eqref{eq: circ monge map} defines an optimal circular transportation map. (We will use the notation $M_\mu^\nu$ for the Monge map from $\mu$ to $\nu$.)
    \item If $\alpha_{\mu,\nu}$ minimizes \eqref{eq: cot alpha}, then
    \begin{equation}\label{eq: mong with alpha}
    M_{\mu}^{\nu}(x)=F_\nu^{-1}(F_\mu(x)-\alpha_{\mu,\nu})
\end{equation}
    \item    If $x_0,x_1$ are two minimizers of \eqref{eq: cot solution x0}, then
    \begin{equation*}
    F_{\nu,x_0}^{-1}\circ F_{\mu,x_0}(x-x_0)+x_0=    F_{\nu,x_1}^{-1}\circ F_{\mu,x_1}(x-x_1)+x_1 \qquad \forall \, x\in [0,1).
    \end{equation*} 
    \item 
    The optimal map defined by the formula \eqref{eq: circ monge map} is unique. (The uniqueness is as functions on $\mathbb{S}^1$, or as functions on $\mathbb{R}$ up to modulo $\mathbb{Z}$).
    \item If also $\nu$ does not give mass to atoms, then $(M_\mu^\nu)^{-1}=M_\nu^\mu$.  
\end{enumerate}    
\end{theorem}

Having established the necessary background, we are now poised to introduce our proposed metric. In the subsequent section, we present the Linear Circular Optimal Transport (LCOT) metric.

\section{Method}
\label{sec: method}

\subsection{Linear Circular Optimal Transport Embedding (LCOT)}

By following the footsteps of\cite{wang2013linear}, starting from the COT framework, we will define an embedding for circular measures by computing the optimal displacement from a  fixed reference measure. Then, the $L^p$-distance on the embedding space defines a new distance between circular measures (Theorem \ref{thm: distance} below).

\begin{definition}[LCOT Embedding] \label{def: embedding} 

For a fixed reference measure $\mu\in\mathcal{P}(\mathbb{S}^1)$ that is absolutely continuous with respect to the Lebesgue measure on $\mathbb{S}^1$, we define the \textit{Linear Circular Optimal Transport (LCOT) Embedding} of a target measure $\nu\in\mathcal{P}(\mathbb{S}^1)$ with respect to the cost $COT_{h}(\cdot,\cdot)$, for a convex increasing function $h:\mathbb{R}\to \mathbb{R}_+$, 
by
    \begin{equation}\label{eq: embedding general}
        \widehat{\nu}^{\mu,h}(x):= F_{\nu,x_{cut}}^{-1}(F_\mu(x-x_{cut}))-(x-x_{cut})=F_\nu^{-1}(F_\mu(x)-\alpha_{\mu,\nu})-x, \quad x\in[0,1), 
    \end{equation}
    where $x_{cut}$ is any minimizer of \eqref{eq: cot solution x0} and $\alpha_{\mu,\nu}$ is the minimizer of \eqref{eq: cot alpha}.
\end{definition}

The LCOT-Embedding corresponds to the optimal (circular) displacement that comes from the problem of transporting the reference measure $\mu$ to the given target measure $\nu$ with respect to a general cost $COT_{h}(\cdot,\cdot)$ (see \eqref{eq: mong with alpha} from Theorem \ref{thm: cot embedding} and \eqref{eq: circ opt disp}).

\begin{definition}[LCOT distance]
Under the settings of Definition \ref{def: embedding}, we define the LCOT-discrepancy  by
    \begin{align*}
        LCOT_{\mu,h}({\nu_1},{\nu_2}) &:=
\int_{0}^1h\left(\min_{k\in\mathbb{Z}}\{|\widehat{\nu_1}^{\mu,h}(t)-\widehat{\nu_2}^{\mu,h}(t)+k|_\mathbb{R}\}\right) \, d\mu(t), \nonumber \quad \forall \nu_1,\nu_2\in \mathcal{P}(\mathbb{S}^1).
    \end{align*}

In particular, when $h(\cdot)=|\cdot|^p$, for $1\leq p<\infty$, we define
\begin{align*}
        LCOT_{\mu,p}({\nu_1},{\nu_2}) &:=\|\widehat{\nu_1}^{\mu,h}-\widehat{\nu_2}^{\mu,h}\|^p_{L^p(\mathbb{S}^1,d\mu)} \\
&=
\int_{0}^1\left(\min_{ k\in\mathbb{Z}}\{|\widehat{\nu_1}^{\mu,h}(t)-\widehat{\nu_2}^{\mu,h}(t)+k|_\mathbb{R}\}\right)^p \,  d\mu(t), \nonumber
    \end{align*}
where
\begin{equation}\label{eq: embedding space}
  L^p(\mathbb{S}^1,d\mu):=\{f: \mathbb{S}^1 \to \mathbb{R} \mid \quad \|f\|_{L^p(\mathbb{S}^1,d\mu)}:=\left(\int_{\mathbb{S}^1}|f(t)|_{\mathbb{S}^1}^p \, d\mu(t)\right)^{1/p} <\infty\}.  
\end{equation}
If $\mu=Unif(\mathbb{S}^1)$, we use the notation $L^p(\mathbb{S}^1):=L^p(\mathbb{S}^1,d\mu)$.
\end{definition}

The embedding $ \nu \mapsto \widehat{\nu} $ as outlined by  \eqref{eq: embedding general} is consistent with the definition of the Logarithm function given in \cite[Definition 2.7]{sarrazin2023linearized} (we also refer to \cite{wang2013linear} for the LOT framework). However, the emphasis of the embedding in this paper is on computational efficiency, and a closed-form solution is provided. Additional details are available in Appendix \ref{sec: appendix log exp}.

\begin{remark}
If the reference measure is $\mu= Unif(\mathbb{S}^1)$, given a target measure
$\nu\in \mathcal{P}(\mathbb{S}^1)$, we denote the LCOT-Embedding $\widehat{\nu}^{\mu,h}$ of $\nu$ with respect to the cost $COT_2(\cdot,\cdot)$ (i.e., $h(x)=|x|^2$) simply by $\widehat{\nu}$. Due to Theorem \ref{thm: cot embedding} and \eqref{eq: alpha mean}, the expression \ref{eq: embedding general} reduces to
    \begin{equation}\label{eq: embedding}
        \widehat{\nu}(x) := F_\nu^{-1}\left(x-\mathbb{E}(\nu)+\frac{1}{2}\right)-x, \qquad x\in [0,1).
    \end{equation}
In this setting, we denote $LCOT_{\mu,h}(\cdot,\cdot)$ simply by $LCOT(\cdot,\cdot)$. That is, given $\nu_1,\nu_2\in\mathcal{P}(\mathbb{S}^1)$, \begin{align}\label{eq: LCOT1}
LCOT({\nu_1},{\nu_2}) &:=\|\widehat{\nu_1}-\widehat{\nu_2}\|_{L^2(\mathbb{S}^1)}^2 
=
\int_{0}^1\left(\min_{k\in\mathbb{Z}}\{|\widehat{\nu_1}(t)-\widehat{\nu_2}(t)+k|_\mathbb{R}\}\right)^2\, dt . 
\end{align}

All our experiments are performed using the embedding $\widehat{\nu}$ given by \ref{eq: embedding} due to the robustness of the closed-form formula \ref{eq: alpha mean} for the minimizer $\alpha_{\mu,\nu}$ of \eqref{eq: cot alpha} when $h(x)=|x|^2$ and $\mu= Unif(\mathbb{S}^1)$. 
\end{remark}

\begin{remark}\label{remark: COT vs LCOT}
Let $\mu\in\mathcal{P}(\mathbb{S}^1)$ be absolutely continuous with respect to the Lebesgue measure on $\mathbb{S}^1$. Given $\nu\in\mathcal{P}(\mathbb{S}^1)$. 
\begin{align*}
  COT_{h}(\mu,\nu) &= 
\int_{0}^1 h\left(|\widehat{\nu}^{\mu,h}(t)|_{\mathbb{S}^1}\right)\, dt
=\int_{0}^1 h\left(|\widehat{\nu}^{\mu,h}(t)-\widehat{\mu}^{\mu,h}(t)|_{\mathbb{S}^1}\right) \, dt
=LCOT_{\mu,h}(\mu,\nu).  
\end{align*}

In particular,
$$COT_2(\mu,\nu) = \|\widehat{\nu}\|^2_{L^2(\mathbb{S}^1)}=\|\widehat{\nu}-\widehat{\mu}\|^2_{L^2(\mathbb{S}^1)}=LCOT(\mu,\nu).$$
\end{remark}

\begin{proposition}[Invertibility of the LCOT-Embedding.]
    Let $\mu\in \mathcal{P}(\mathbb{S}^1)$ be absolutely continuous with respect to the Lebesgue measure on $\mathbb{S}^1$, and let $\nu\in \mathcal{P}(\mathbb{S}^1)$. Then,
    \begin{equation*}\label{eq: inverse transform}
        \nu=(\widehat{\nu}^{\mu,h}+\mathrm{id})_\#\mu.
    \end{equation*}
\end{proposition}

We refer to Prop. \ref{prop: properties of embedding} in the Appendix for more properties of the LCOT-Embedding.

\begin{theorem}\label{thm: distance}
Let $\mu\in \mathcal{P}(\mathbb{S}^1)$ be absolutely continuous with respect to the Lebesgue measure on $\mathbb{S}^1$, and let $h(x)=|x|^p$, for $1\leq p<\infty$. 
Then 
    $LCOT_{\mu,p}(\cdot,\cdot)^{1/p}$ is a distance on $\mathcal{P}(\mathbb{S}^1)$. In particular,     $LCOT(\cdot,\cdot)^{1/2}$ is a distance on $\mathcal{P}(\mathbb{S}^1)$.
\end{theorem}

\subsection{LCOT interpolation between circular measures}

Given a COT Monge map and the LCOT embedding, we can compute a linear interpolation between circular measures (refer to \cite{wang2013linear} for a similar approach on the Euclidean setting). First, for arbitrary measures $\sigma,\nu \in \mathcal{P}(\mathbb{S}^1)$ the COT interpolation can be written as:
\begin{equation}\label{eq: interpol COT}
        \rho^{COT}_t:=\left((1-t)\mathrm{id}+tM_{\sigma}^\nu\right)_\#\sigma, \qquad t\in[0,1].
    \end{equation}
Similarly, for a fixed reference measure $\mu\in \mathcal{P}(\mathbb{S}^1)$, we can write the LCOT interpolation as:
\begin{equation}\label{eq: interpol LCOT}
        \rho^{LCOT}_t :=\left((1-t)(\widehat{\sigma}+\mathrm{id})+t(\widehat{\nu}+\mathrm{id})\right)_{\#}\mu, \qquad t\in [0,1],
    \end{equation}
where we have $\rho^{COT}_{t=0}=\rho^{LCOT}_{t=0}=\sigma$ and $\rho^{COT}_{t=1}=\rho^{LCOT}_{t=1}=\nu$. In Figure \ref{fig:LCOT}, we show such interpolations between the reference measure $\mu$ and two arbitrary measures $\nu_1$ and $\nu_2$ for COT and LCOT. As can be seen, the COT and LCOT interpolations between $\mu$ and $\nu_i$s coincide (by definition), while the interpolation between $\nu_1$ and $\nu_2$ is different for the two methods. We also provide an illustration of the logarithmic and exponential maps to, and from, the LCOT embedding.

\begin{figure}[ht]
    \centering
    \includegraphics[width=1\linewidth]{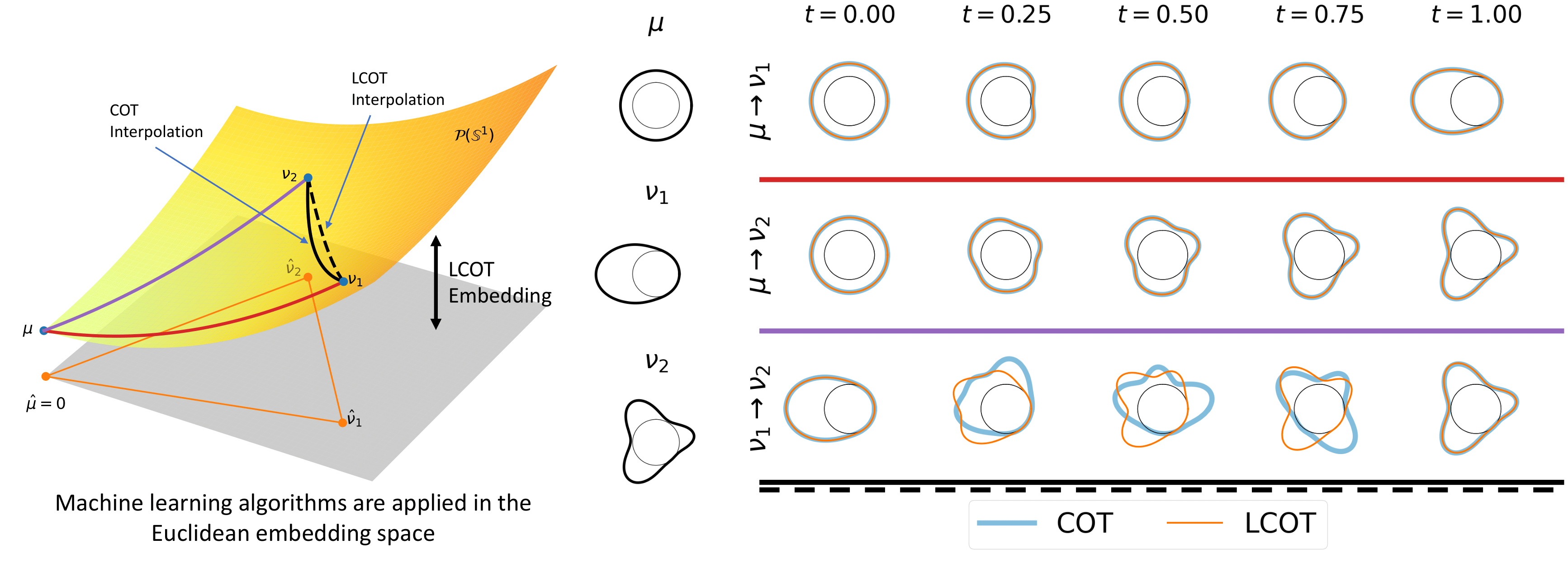}
    \caption{Left: Illustration of the LCOT embedding, the linearization process (logarithmic map), and measure interpolations. Right: Pairwise interpolations between reference measure $\mu$ and measures $\nu_1$ and $\nu_2$, using formulas in \ref{eq: interpol COT} (COT) and \ref{eq: interpol LCOT} (LCOT).}
    \label{fig:LCOT}
\end{figure}
\newpage

\subsection{Time complexity of Linear COT distance between discrete measures}
\label{subsec: time}


According to \cite[Theorem 6.2]{delon2010fast}, for discrete measures $\nu_1,\nu_2$ with $N_1, N_2$ sorted points, the \emph{binary search} algorithm requires $\mathcal{O}((N_1+N_2)\log (1/\epsilon))$ computational time to find an $\epsilon$-approximate solution for $\alpha_{\nu_1,\nu_2}$. If $M$ is the least common denominator for all probability masses, an exact solution can be obtained in $\mathcal{O}((N_1+N_2) \ln M)$. Then, for a given $\epsilon>0$ and $K$ probability measures, $\{\nu_k\}_{k=1}^K$, each with $N$ points, the total time to pairwise compute the COT distance is $\mathcal{O}(K^2N\ln(1/\epsilon))$. 
For LCOT, when the reference $\mu$ is the Lebesgue measure, the optimal $\alpha_{\mu,\nu_k}$ has a closed-form solution (see  \eqref{eq: alpha mean}) and the time complexity for computing the LCOT embedding via  \eqref{eq: embedding} is $\mathcal{O}(N)$. 
The LCOT distance calculation between $\nu_i$ and $\nu_j$ according to  \eqref{eq: LCOT1} requires $\mathcal{O}(N)$ computations. Hence, the total time for pairwise LCOT distance computation between $K$ probability measures, $\{\nu_k\}_{k=1}^K$, each with $N$ points, would be $\mathcal{O}(K^2N+KN)$. See Appendix \ref{sec: appendix numerics} for further explanation.

To verify these time complexities, we evaluate the computational time for COT and LCOT algorithms and present the results in Figure \ref{fig:time_plot}. We generate $K$ random discrete measures, $\{\nu_k\}_{k=1}^K$, each with $N$ samples, and for the reference measure, $\mu$, we choose: 1) the uniform discrete measure, and 2) a random discrete measure, both with $N_0=N$ samples. To calculate $\alpha_{\mu,\nu_k}$, we considered the two scenarios, using the binary search \cite{delon2010fast} for the non-uniform reference, and using \eqref{eq: alpha mean} for the uniform reference. We labeled them as, ``uniform ref.'' and ``non-uniform ref.'' Then, in our first experiment, we set $K=2$ and measured the wall-clock time for calculating COT and LCOT while varying $N\in\{100,200,\ldots, 20000\}$. For our second experiment, and for $N\in\{500,1000,5000\}$, we vary $K\in\{2,4,6,\ldots, 64\}$ and measure the total time for calculating pairwise COT and LCOT distances. The computational benefit of LCOT is evident from Figure \ref{fig:time_plot}.

\begin{figure}[ht]
    \centering
    \includegraphics[width=0.95\textwidth]{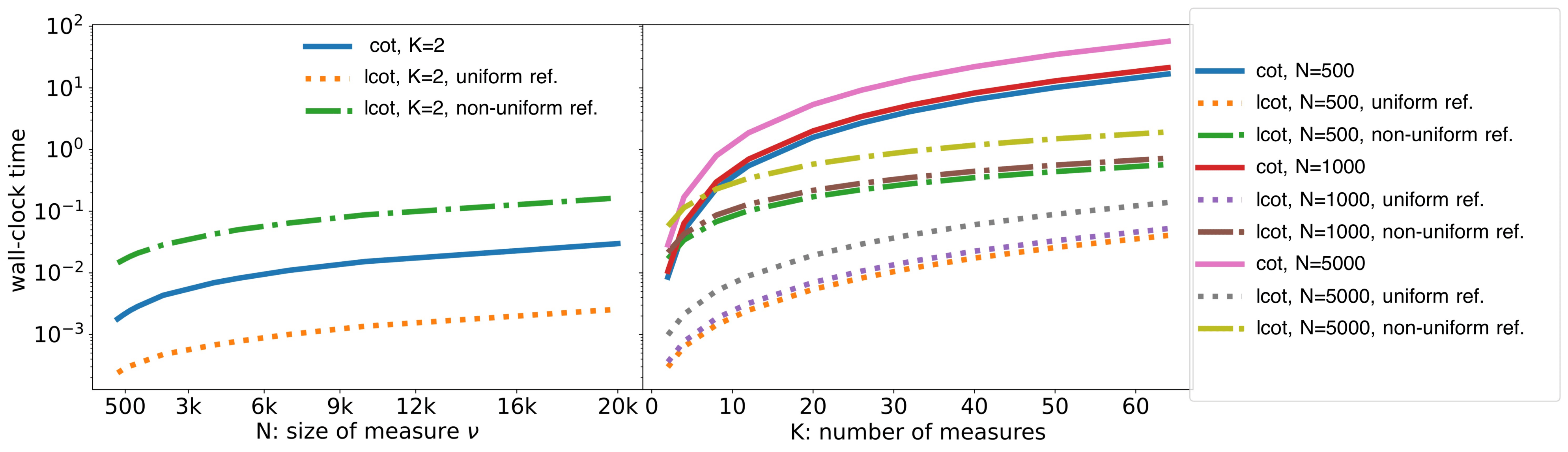}    
    \caption{Computational time analysis of COT and LCOT, for pairwise comparison of $K$ discrete measures, each with $N$ samples. Left: Wall-clock time for $K=2$ and $N\in\{500, 1000, \ldots, 5000\}$. Right: Wall-clock time for $N\in\{500, 1000, 5000\}$, and $K\in \{2,4,6,\ldots,64\}$. Solid lines are COT, dotted are LCOT with a uniform reference and dash-dotted are LCOT with a non-uniform reference.}
    \label{fig:time_plot}
\end{figure}

\section{Experiments}
\label{sec: experiments}


To better understand the geometry induced by the LCOT metric, we perform Multidimensional Scaling (MDS) \cite{kruskal1964multidimensional} on a family of densities, where the discrepancy matrices are computed using LCOT, COT, OT (with a fixed cutting point), and the Euclidean distance. 

\smallskip

{\bf Experiment 1.} We generate three families of circular densities, calculate pairwise distances between them, and depict their MDS embedding in Figure \ref{fig: mds 3 classes}. In short, the densities are chosen as follows; we start with two initial densities: (1) a von Mises centered at the south pole of the circle ($\mu$=0.5), (2) a bimodal von Mises centered at the east ($\mu$=0.25) and west ($\mu$=0.75) ends of the circle. Then, we create 20 equally distant circular translations of each of these densities to capture the geometry of the circle. Finally, we parametrize the COT geodesic between the initial densities and generate 20 extra densities on the geodesic. Figure \ref{fig: mds 3 classes} shows these densities in green, blue, and red, respectively. 
The representations given by the MDS visualizations show that LCOT and COT capture the geometry of the circle coded in the translation property in an intuitive fashion. In contrast, OT and Euclidean distances do not capture the underlying geometry of the problem. 

\smallskip

{\bf Experiment 2.} To assess the separability properties of the LCOT embedding, we follow a similar experiment design as in \cite{landler2021advice}. 
We consider six groups of circular density functions as in the third row of Figure \ref{fig: mds 3 classes}: unimodal von Mises (axial: 0$^{\circ}$), wrapped skew-normal, symmetric bimodal von Mises (axial: 0$^{\circ}$ and 180$^{\circ}$), asymmetric bimodal von Mises (axial: 0$^{\circ}$ and 120$^{\circ}$), symmetric trimodal von Mises (axial: 0$^{\circ}$, 120$^{\circ}$ and 240$^{\circ}$), asymmetric trimodal von Mises (axial: 0$^{\circ}$, 240$^{\circ}$ and 225$^{\circ}$). 
We assign a von Mises distribution with a small spread ($\kappa=200$) to each distribution's axis/axes to introduce random perturbations of these distributions.  
We generate 20 sets of simulated densities and sample each with 50-100 samples. Following the computation of pairwise distances among the sets of samples using LCOT, COT, OT, and Euclidean methods, we again employ MDS to visualize the separability of each approach across the six circular density classes mentioned above. The outcomes are presented in the bottom row of Figure \ref{fig: mds 3 classes}. It can be seen that LCOT stands out for its superior clustering outcomes, featuring distinct boundaries between the actual classes, outperforming the other methods.

\smallskip

{\bf Experiment 3.} In our last experiment, we consider the calculation of the barycenter of circular densities. Building upon Experiments 1 and 2, we generated unimodal, bimodal, and trimodal von Mises distributions. For each distribution's axis/axes, we assigned a von Mises distribution with a small spread ($\kappa=200$) to introduce random perturbations. These distributions are shown in Figure \ref{fig: bary} (left). Subsequently, we computed both the Euclidean average of these densities and the LCOT barycenter. Notably, unlike COT, the invertible nature of the LCOT embedding allows us to directly calculate the barycenter as the inverse of the embedded distributions' average. The resulting barycenters are illustrated in Fig. \ref{fig: bary}. As observed, the LCOT method accurately recovers the correct barycenter without necessitating additional optimization steps.


\begin{figure}[ht]
    \centering
    \includegraphics[width=1\linewidth]{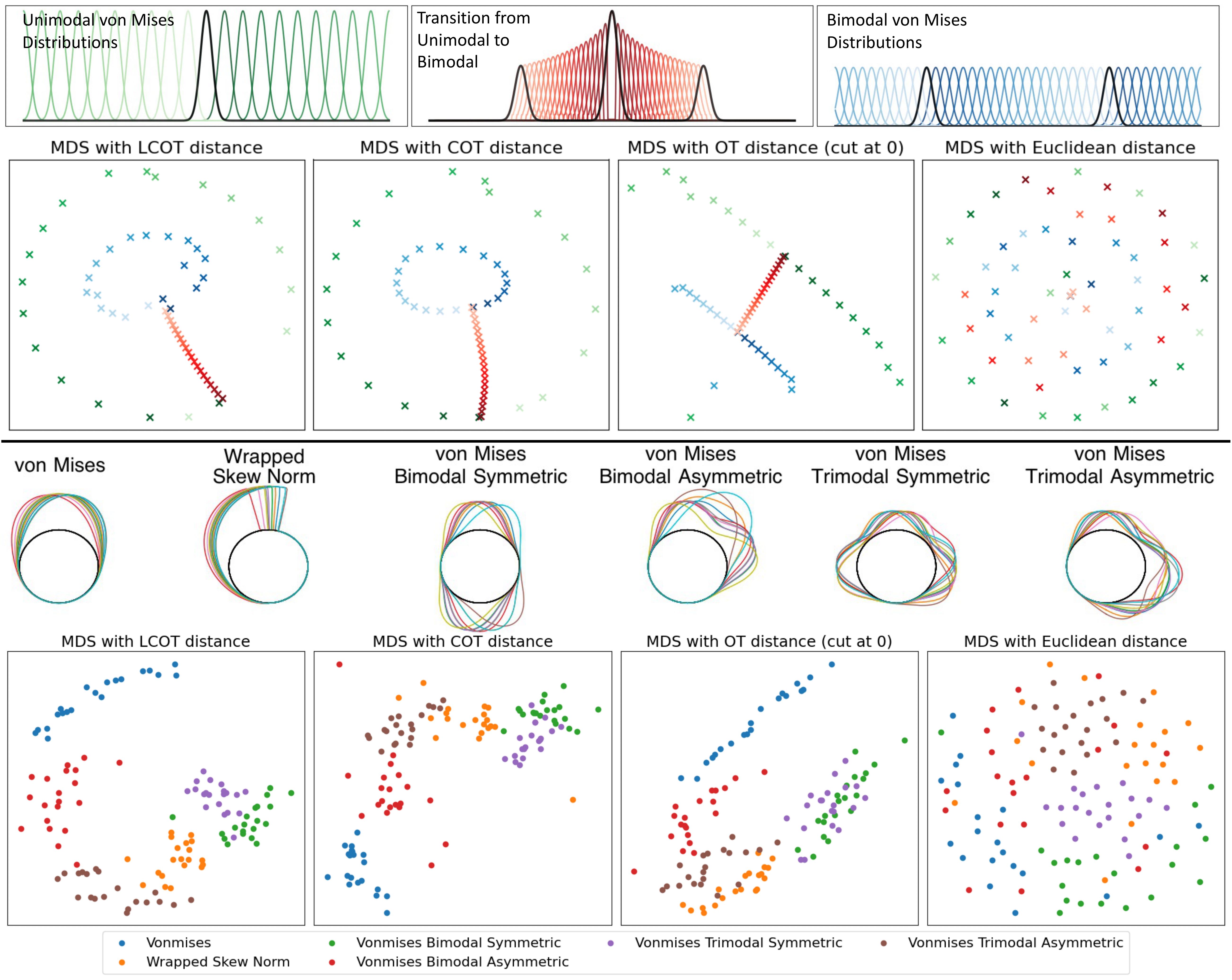}
    \caption{MDS for embedding classes of probability densities into an Euclidean space of dimension $2$ where the original pair-wise distances (COT-distance, LOT-distance, Euclidean or $L^2$-distance) are preserved as well as possible. }
    \label{fig: mds 3 classes}
\end{figure}

\begin{figure}[ht]
    \centering
    \includegraphics[width=\linewidth]{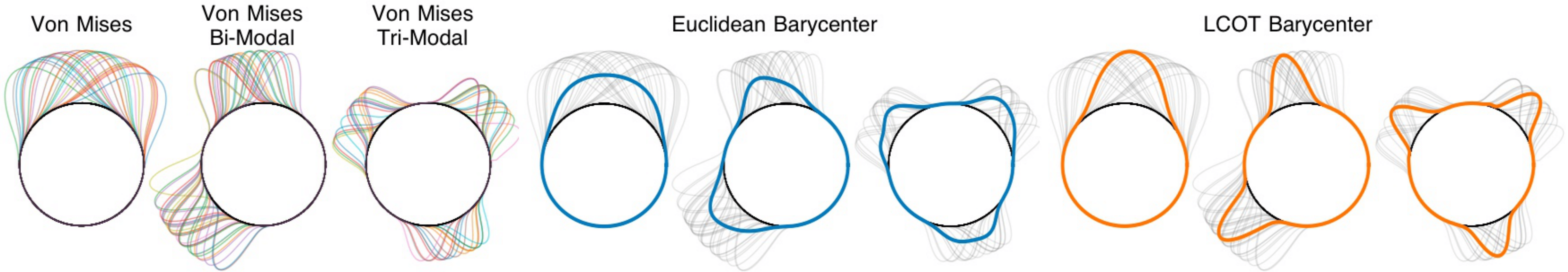}
    \caption{The LCOT barycenter compared to the Euclidean mean.}
    \label{fig: bary}
\end{figure}


\section{Conclusion and discussion}\label{sec: conclusion}

In this paper, we present the Linear Circular Optimal Transport (LCOT) distance, a new metric for circular measures derived from the Linear Optimal Transport (LOT) framework \cite{wang2013linear,kolouri2016continuous,park2018cumulative,cai2022linearized,aldroubi2021signed,moosmuller2023linear}. The LCOT offers 1) notable computational benefits over the COT metric, particularly in pairwise comparisons of numerous measures, and 2) a linear embedding where the $\|\cdot\|_{L^2(\mathbb{S}^1)}$ between embedded distributions equates to the LCOT metric. We consolidated scattered results on circular OT into Theorem \ref{thm: cot embedding} and introduced the LCOT metric and embedding, validating LCOT as a metric in Theorem \ref{thm: distance}. In Section \ref{subsec: time}, we assess LCOT's computational complexity for pairwise comparisons of $K$ circular measures, juxtaposing it with COT. We conclude by showcasing LCOT's empirical strengths via MDS embeddings on circular densities using different metrics.

\section*{Acknowledgement}
SK acknowledges partial support from the Defense Advanced
Research Projects Agency (DARPA) under Contract No. HR00112190135 and HR00112090023, and the Wellcome LEAP Foundation. GKR acknowledges support from ONR N000142212505, and NIH GM130825.

\bibliography{ref}

\begin{thebibliography}{10}\itemsep=-1pt

\bibitem{aldroubi2021partitioning}
Akram Aldroubi, Shiying Li, and Gustavo~K Rohde.
\newblock Partitioning signal classes using transport transforms for data analysis and machine learning.
\newblock {\em Sampling theory, signal processing, and data analysis}, 19(1):6, 2021.

\bibitem{aldroubi2021signed}
Akram Aldroubi, Rocio~Diaz Martin, Ivan Medri, Gustavo~K Rohde, and Sumati Thareja.
\newblock The signed cumulative distribution transform for 1-d signal analysis and classification.
\newblock {\em Foundations of Data Science}, 4:137--163, 2022.

\bibitem{alvarez2020geometric}
David Alvarez-Melis and Nicolo Fusi.
\newblock Geometric dataset distances via optimal transport.
\newblock {\em Advances in Neural Information Processing Systems}, 33:21428--21439, 2020.

\bibitem{arjovsky2017wasserstein}
Martin Arjovsky, Soumith Chintala, and L{\'e}on Bottou.
\newblock {Wasserstein} generative adversarial networks.
\newblock In {\em International conference on machine learning}, pages 214--223. PMLR, 2017.

\bibitem{bai2022sliced}
Yikun Bai, Bernhard Schmitzer, Mathew Thorpe, and Soheil Kolouri.
\newblock Sliced optimal partial transport.
\newblock {\em arXiv preprint arXiv:2212.08049}, 2022.

\bibitem{belongie2000shape}
Serge Belongie, Jitendra Malik, and Jan Puzicha.
\newblock Shape context: A new descriptor for shape matching and object recognition.
\newblock {\em Advances in neural information processing systems}, 13, 2000.

\bibitem{bonet2022spherical}
Cl{\'e}ment Bonet, Paul Berg, Nicolas Courty, Fran{\c{c}}ois Septier, Lucas Drumetz, and Minh-Tan Pham.
\newblock Spherical sliced-wasserstein.
\newblock {\em ICLR}, 2023.

\bibitem{cabrelli1998linear}
Carlos~A Cabrelli and Ursula~M Molter.
\newblock A linear time algorithm for a matching problem on the circle.
\newblock {\em Information processing letters}, 66(3):161--164, 1998.

\bibitem{cai2022linearized}
Tianji Cai, Junyi Cheng, Bernhard Schmitzer, and Matthew Thorpe.
\newblock The linearized hellinger--kantorovich distance.
\newblock {\em SIAM Journal on Imaging Sciences}, 15(1):45--83, 2022.

\bibitem{cloninger2023linearized}
Alexander Cloninger, Keaton Hamm, Varun Khurana, and Caroline Moosm{\"u}ller.
\newblock Linearized wasserstein dimensionality reduction with approximation guarantees.
\newblock {\em arXiv preprint arXiv:2302.07373}, 2023.

\bibitem{courty2017joint}
Nicolas Courty, R{\'e}mi Flamary, Amaury Habrard, and Alain Rakotomamonjy.
\newblock Joint distribution optimal transportation for domain adaptation.
\newblock {\em Advances in Neural Information Processing Systems}, 30, 2017.

\bibitem{damodaran2018deepjdot}
Bharath~Bhushan Damodaran, Benjamin Kellenberger, R{\'e}mi Flamary, Devis Tuia, and Nicolas Courty.
\newblock Deepjdot: Deep joint distribution optimal transport for unsupervised domain adaptation.
\newblock In {\em Proceedings of the European conference on computer vision (ECCV)}, pages 447--463, 2018.

\bibitem{delon2010fast}
Julie Delon, Julien Salomon, and Andrei Sobolevski.
\newblock Fast transport optimization for monge costs on the circle.
\newblock {\em SIAM Journal on Applied Mathematics}, 70(7):2239--2258, 2010.

\bibitem{frogner2015learning}
Charlie Frogner, Chiyuan Zhang, Hossein Mobahi, Mauricio Araya, and Tomaso~A Poggio.
\newblock Learning with a {Wasserstein} loss.
\newblock {\em Advances in neural information processing systems}, 28, 2015.

\bibitem{haker2004optimal}
Steven Haker, Lei Zhu, Allen Tannenbaum, and Sigurd Angenent.
\newblock Optimal mass transport for registration and warping.
\newblock {\em International Journal of computer vision}, 60:225--240, 2004.

\bibitem{ho2017multilevel}
Nhat Ho, XuanLong Nguyen, Mikhail Yurochkin, Hung~Hai Bui, Viet Huynh, and Dinh Phung.
\newblock Multilevel clustering via wasserstein means.
\newblock In {\em International conference on machine learning}, pages 1501--1509. PMLR, 2017.

\bibitem{hundrieser2022statistics}
Shayan Hundrieser, Marcel Klatt, and Axel Munk.
\newblock The statistics of circular optimal transport.
\newblock In {\em Directional Statistics for Innovative Applications: A Bicentennial Tribute to Florence Nightingale}, pages 57--82. Springer, 2022.

\bibitem{jammalamadaka2001topics}
S~Rao Jammalamadaka and Ashis SenGupta.
\newblock {\em Topics in circular statistics}, volume~5.
\newblock world scientific, 2001.

\bibitem{khamis2023earth}
Abdelwahed Khamis, Russell Tsuchida, Mohamed Tarek, Vivien Rolland, and Lars Petersson.
\newblock Earth movers in the big data era: A review of optimal transport in machine learning.
\newblock {\em arXiv preprint arXiv:2305.05080}, 2023.

\bibitem{kolouri2017optimal}
Soheil Kolouri, Se~Rim Park, Matthew Thorpe, Dejan Slepcev, and Gustavo~K Rohde.
\newblock Optimal mass transport: Signal processing and machine-learning applications.
\newblock {\em IEEE signal processing magazine}, 34(4):43--59, 2017.

\bibitem{kolouri2018sliced2}
Soheil Kolouri, Phillip~E Pope, Charles~E Martin, and Gustavo~K Rohde.
\newblock Sliced wasserstein auto-encoders.
\newblock In {\em International Conference on Learning Representations}, 2018.

\bibitem{kolouri2016continuous}
Soheil Kolouri, Akif~B Tosun, John~A Ozolek, and Gustavo~K Rohde.
\newblock A continuous linear optimal transport approach for pattern analysis in image datasets.
\newblock {\em Pattern recognition}, 51:453--462, 2016.

\bibitem{kruskal1964multidimensional}
Joseph~B Kruskal.
\newblock Multidimensional scaling by optimizing goodness of fit to a nonmetric hypothesis.
\newblock {\em Psychometrika}, 29(1):1--27, 1964.

\bibitem{landler2021advice}
Lukas Landler, Graeme~D Ruxton, and E~Pascal Malkemper.
\newblock Advice on comparing two independent samples of circular data in biology.
\newblock {\em Scientific reports}, 11(1):20337, 2021.

\bibitem{le2023diffeomorphic}
Tung Le, Khai Nguyen, Shanlin Sun, Kun Han, Nhat Ho, and Xiaohui Xie.
\newblock Diffeomorphic deformation via sliced wasserstein distance optimization for cortical surface reconstruction.
\newblock {\em arXiv preprint arXiv:2305.17555}, 2023.

\bibitem{levine2002signal}
Joel~D Levine, Pablo Funes, Harold~B Dowse, and Jeffrey~C Hall.
\newblock Signal analysis of behavioral and molecular cycles.
\newblock {\em BMC neuroscience}, 3(1):1--25, 2002.

\bibitem{liu2022wasserstein}
Xinran Liu, Yikun Bai, Yuzhe Lu, Andrea Soltoggio, and Soheil Kolouri.
\newblock Wasserstein task embedding for measuring task similarities.
\newblock {\em arXiv preprint arXiv:2208.11726}, 2022.

\bibitem{lowe2004distinctive}
David~G Lowe.
\newblock Distinctive image features from scale-invariant keypoints.
\newblock {\em International journal of computer vision}, 60:91--110, 2004.

\bibitem{mccann2001polar}
Robert~J McCann.
\newblock Polar factorization of maps on riemannian manifolds.
\newblock {\em Geometric \& Functional Analysis GAFA}, 11(3):589--608, 2001.

\bibitem{moosmuller2023linear}
Caroline Moosm{\"u}ller and Alexander Cloninger.
\newblock Linear optimal transport embedding: provable wasserstein classification for certain rigid transformations and perturbations.
\newblock {\em Information and Inference: A Journal of the IMA}, 12(1):363--389, 2023.

\bibitem{mukherjee2021end}
Subhadip Mukherjee, Marcello Carioni, Ozan {\"O}ktem, and Carola-Bibiane Sch{\"o}nlieb.
\newblock End-to-end reconstruction meets data-driven regularization for inverse problems.
\newblock {\em Advances in Neural Information Processing Systems}, 34:21413--21425, 2021.

\bibitem{park2018cumulative}
Se~Rim Park, Soheil Kolouri, Shinjini Kundu, and Gustavo~K Rohde.
\newblock The cumulative distribution transform and linear pattern classification.
\newblock {\em Applied and computational harmonic analysis}, 45(3):616--641, 2018.

\bibitem{peyre2019computational}
Gabriel Peyr{\'e}, Marco Cuturi, et~al.
\newblock Computational optimal transport: With applications to data science.
\newblock {\em Foundations and Trends{\textregistered} in Machine Learning}, 11(5-6):355--607, 2019.

\bibitem{rabin2011}
Julien Rabin, Julie Delon, and Yann Gousseau.
\newblock Transportation distances on the circle.
\newblock {\em Journal of Mathematical Imaging and Vision}, 41:147–--167, 2011.

\bibitem{sant2015}
F. Santambrogio.
\newblock {\em Optimal Transport for Applied Mathematicians. Calculus of Variations, PDEs and Modeling}.
\newblock Birkhäuser, 2015.

\bibitem{sarrazin2023linearized}
Cl{\'e}ment Sarrazin and Bernhard Schmitzer.
\newblock Linearized optimal transport on manifolds.
\newblock {\em arXiv preprint arXiv:2303.13901}, 2023.

\bibitem{tolstikhin2017wasserstein}
Ilya Tolstikhin, Olivier Bousquet, Sylvain Gelly, and Bernhard Schoelkopf.
\newblock {Wasserstein} auto-encoders.
\newblock {\em arXiv preprint arXiv:1711.01558}, 2017.

\bibitem{twiss1992structural}
Robert~J Twiss and Eldridge~M Moores.
\newblock {\em Structural geology}.
\newblock Macmillan, 1992.

\bibitem{villani2009optimal}
Cedric Villani.
\newblock {\em Optimal transport: old and new}.
\newblock Springer, 2009.

\bibitem{wang2013linear}
Wei Wang, Dejan Slep{\v{c}}ev, Saurav Basu, John~A Ozolek, and Gustavo~K Rohde.
\newblock A linear optimal transportation framework for quantifying and visualizing variations in sets of images.
\newblock {\em International journal of computer vision}, 101:254--269, 2013.

\end{thebibliography}
\bibliographystyle{ieee_fullname}

\newpage
\appendix
\section{Appendix}\label{sec: Appendix}

\subsection{Proofs}\label{sec: Appendix proofs}

\begin{proof}[Proof of Proposition \ref{prop: equivalent COT}]
  The proof of Proposition \ref{prop: equivalent COT} is provided in \cite{delon2010fast} for the optimal coupling for any pair of probability measures on $\mathbb{S}^1$. For the particular and enlightening case of discrete probability measures on $\mathbb{S}^1$, we refer the reader to \cite{rabin2011}.  

For completeness, notice that the relation between $x_0$ and $\alpha$ hols by changing variables, using 1-periodicity of $\mu$ and $\nu$ and Definition \ref{def: quantile} (see also \cite[Proposition 1]{bonet2022spherical}):
  \begin{align*}
\int_{0}^1h(|F_{\mu,x_0}^{-1}(x)&-F_{\nu,x_0}^{-1}(x)|_{\mathbb{R}}) \, dx\\ &=\int_0^1h(|(F_{\mu}(\cdot+x_0)-F_\mu(x_0))^{-1}(x)-(F_{\nu}(\cdot+x_0)-F_\nu(x_0))^{-1}(x)|_{\mathbb{R}}) \, dx \\
&=\int_0^1h(|(F_{\mu}-F_\mu(x_0))^{-1}(x)-(F_{\nu}-F_\nu(x_0))^{-1}(x)|_{\mathbb{R}}) \, dx \\
&=\int_0^1h(|F_{\mu}^{-1}(x+F_\mu(x_0))-F_{\nu}^{-1}(x+F_\nu(x_0))|_{\mathbb{R}}) \, dx \\
&=\int_0^1h(|F_{\mu}^{-1}(x+\underbrace{F_\mu(x_0)-F_\nu(x_0)}_{\alpha})-F_{\nu}^{-1}(x)|_{\mathbb{R}}) \, dx 
\end{align*}   

In particular, if $h(x)=|x|^2$, and $\mu= Unif(\mathbb{S}^1)$, then 
\begin{align*}
    COT_2(\mu,\nu)&=\inf_{\alpha\in\mathbb{R}}\int_0^1|F_{\mu}^{-1}(x+{\alpha})-F_{\nu}^{-1}(x)|_{\mathbb{R}}^2 \, dx \\
    &=\inf_{\alpha\in\mathbb{R}}\int_0^1|x+{\alpha}-F_{\nu}^{-1}(x)|^2 \, dx \\
    &=\inf_{\alpha\in\mathbb{R}}\left(\int_0^1|F_{\nu}^{-1}(x)-x|^2 \, dx -2\alpha\int_0^1(F_{\nu}^{-1}(x)-x) \, dx +\alpha^2\right)\\&=\inf_{\alpha\in\mathbb{R}}\left(\int_0^1|F_{\nu}^{-1}(x)-x|^2 \, dx -2\alpha\left(\int_0^1 x \, d\nu(x)-\frac{1}{2}\right)  +\alpha^2\right)\\
    &=\inf_{\alpha\in\mathbb{R}}\left(\int_0^1|F_{\nu}^{-1}(x)-x|^2 \, dx -2\alpha\left(\mathbb{E}(\nu)-\frac{1}{2}\right) +\alpha^2\right)\\
    &=\int_0^1|F_{\nu}^{-1}(x)-x|^2 \, dx -2{\underbrace{\left(\mathbb{E}(\nu)-\frac{1}{2}\right)}_{\alpha_{\mu,\nu}} }^2.
\end{align*}
\end{proof}

\newpage

\begin{proof}[Proof of Theorem \ref{thm: cot embedding}]\,
    \begin{enumerate}
        \item        
        First, we will show that the map $M_\mu^\nu$ given by \eqref{eq: circ monge map} satisfies $(M_\mu^\nu)_\#\mu=\nu$. Here $\mu$ and $\nu$ are the extended measures form $\mathbb{S}^1$ to $\mathbb{R}$ having CDFs equal to $F_{\mu}$ and $F_{\nu}$, respectively, defined by \eqref{eq: CDF on [0,1)} and \ref{eq: extended CDF}. By choosing the system of coordinates $\widetilde{x}\in[0,1)$ that starts at $x_{cut}$ (see Figure \ref{fig: coord}) then, 
        $$M_\mu^\nu(\widetilde{x})=F_{\nu,x_{cut}}^{-1}\circ F_{\mu,x_{cut}}( \widetilde{x})$$
        (see \eqref{eq: circ monge map from x_0}). Let $\mu_{x_{cut}}$ and $\nu_{x_{cut}}$ be the (1-periodic) measures on $\mathbb{R}$ having CDFs $F_{\mu,x_{cut}}$ and $F_{\nu,x_{cut}}$, respectively, i.e., $F_{\nu,x_{cut}}(\widetilde{x})=\mu_{x_{cut}}([0,\widetilde{x}))$ (analogously for $\nu_{x_{cut}}$). That is, we have unrolled $\mu$ and $\nu$ from $\mathbb{S}^1$ to $\mathbb{R}$, where the origin $0\in\mathbb{R}$ corresponds to $x_{cut}\in\mathbb{S}^1$ (see Figure \ref{fig: cut_pdf}). Thus, a classic computation yields 
        $$(F_{\nu,x_{cut}}^{-1}\circ F_{\mu,x_{cut}})_\#\mu_{x_{cut}}=(F_{\nu,x_{cut}}^{-1})_\# \left( (F_{\mu,x_{cut}})_\# \mu_{x_{cut}}\right)=(F_{\nu,x_{cut}}^{-1})_\#\mathcal{L}_{\mathbb{S}^1}=\nu_{x_{cut}}$$
       where $\mathcal{L}_{\mathbb{S}^1}= Unif(\mathbb{S}^1)$ denotes the Lebesgue 
       measure on the circle. We used that $(F_\mu)_\#\mu=\mathcal{L}_{\mathbb{S}^1}$ as $\mu$ does not
        give mass to atoms, and so, if we change the system of coordinates we also have $(F_{\mu,x_{cut}})_\# \mu_{x_{cut}}=\mathcal{L}_{\mathbb{S}^1}$.

        Finally, we have to switch coordinates. Let $$z(\widetilde{x}):=\widetilde{x}+x_{cut}$$ (that is, $z(\widetilde{x})=x$). To visualize this, see Figure \ref{fig: coord}. It holds that 
        \begin{equation}\label{eq: pushforward change coord}
          z_{\#}\nu_{cut}=\nu  
        \end{equation}
        (where we recall that $\nu$ is the extended measure form $\mathbb{S}^1$ to $\mathbb{R}$ having CDF equal to $F_{\mu}$ as in \eqref{eq: CDF on [0,1)} and \ref{eq: extended CDF}). Let us check this fact for intervals:
        \begin{align*}
            z_\#\nu_{x_{cut}}([a,b])&=\nu_{x_{cut}}(z^{-1}([a,b]))=\nu([z^{-1}(a),z^{-1}(b)])\\&=\nu_{x_{cut}}([a-x_{cut},b-x_{cut}])\\
            &=F_{\nu,{x_{cut}}}(b-x_{cut})-F_{\nu, {x_{cut}}}(a-x_{cut})\\
            &=F_{\nu}(b)-F_\nu(x_{cut})-(F_{\nu}(a)-F_\nu(x_{cut}))\\
            &=F_{\nu}(b)-F_{\nu}(a)\\
            &=\nu([a,b]).
        \end{align*}

        \begin{figure}[ht]
            \centering
            \includegraphics[width=0.4\linewidth]{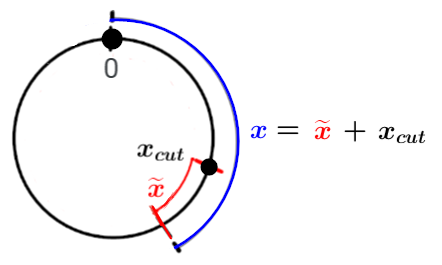}
            \caption{The unit circle (black) can be parametrized as $[0,1)$ in many different ways. In the figure, we marked in black the North Pole as $0$. 
            The canonical parametrization of $\mathbb{S}^1$ identifies the North Pole with $0$.
            Then, also in black, we pick a point $x_{cut}$.  
            The distance in blue $x$ that starts at $0$ equals the distance in red $\widetilde{x}$ that starts at $x_{cut}$ plus the corresponding starting point $x_{cut}$. This allows us to visualize the change of coordinates given by \eqref{eq: change or coord}.    }
            \label{fig: coord}
        \end{figure}

        Besides, it holds that 
        \begin{equation}\label{eq: pushforward F_mu,cut}
        F_{\mu,x_{cut}}(\cdot-x_{cut})_\#\mu = Unif(\mathbb{S}^1),  
        \end{equation}
        in the sense that it is the Lebesgue measure on $\mathbb{S}^1$ extended periodically (with period $1$) to the real line, which we denote by $\mathcal{L}_{\mathbb{S}^1}$. Let us sketch the proof for intervals. First, notice that $F_{\mu,x_{cut}}(x-x_{cut})=F_\mu(x)-F_\mu(x_{cut})$ and so its inverse is $y\mapsto F_{\mu}^{-1}(y+x_{cut})$. Therefore,
        \begin{align*}
            (F_{\mu,x_{cut}}(\cdot-x_{cut}))_\#\mu\left([a,b]\right)&=\mu\left([F_{\mu}^{-1}(a+x_{cut}),F_{\mu}^{-1}(b+x_{cut})]\right)\\
            &=F_\mu(F_{\mu}^{-1}(a+x_{cut}))-F_\mu(F_{\mu}^{-1}(b+x_{cut}))=b-a.
        \end{align*}

       Finally,
       \begin{align*}
           (M_\mu^\nu)_\#\mu&=\left(F_{\nu,x_{cut}}^{-1}(F_{\mu,x_{cut}}(\cdot -x_{cut})+x_{cut})\right)_\#\mu\\
           &=\left(z(F_{\nu,{x_{cut}}}^{-1}( F_{\mu,x_{cut}}(\cdot-x_{cut})))\right)_\#\mu\\
            &=z_\#(F_{\nu,{x_{cut}}}^{-1})_\#  ( F_{\mu,x_{cut}}(\cdot-x_{cut}))_\#\mu\\
            &= z_\#(F_{\nu,{x_{cut}}}^{-1})_\#\mathcal{L}_{\mathbb{S}^1} \qquad (\text{by \eqref{eq: pushforward F_mu,cut}})\\
            &=z_\#\nu_{x_{cut}}\\
            &=\nu \qquad (\text{by \eqref{eq: pushforward change coord}}).
       \end{align*}
       
        Now, let us prove that $M_\mu^\nu$ is optimal. 

        First, assume that $\mu$ is absolutely continuous with respect to the Lebesgue measure on $\mathbb{S}^1$ and let $f_\mu$ denote its density function. We will use the change of variables 
        \begin{equation*}
            \begin{cases}
             &u=F_{\mu,x_{cut}}(x-x_{cut})=F_\mu(x)-F_\mu(x_{cut})\\ &du=f_\mu(x) \, dx .  
            \end{cases}
        \end{equation*}
        So,
        \begin{align*}
            \int_0^1 h(|M_\mu^\nu(x)-x|_\mathbb{R}) \, d\mu(x)
            &=\int_0^1 h(|F_{\nu,x_{cut}}^{-1}(F_{\mu,x_{cut}}(x-x_{cut}))-(x-x_{cut})|_\mathbb{R}) \, \underbrace{f_\mu(x)dx}_{d\mu(x)}\\
            &=            
            \int_{-x_0}^{1-x_0} h(|F_{\nu,x_{cut}}^{-1}(u)-F_{\mu,x_{0}}^{-1}(u)|_\mathbb{R}) \, du\\
            &=\int_{0}^{1} h(|F_{\nu,x_{cut}}^{-1}(u)-F_{\mu,x_{cut}}^{-1}(u)|_\mathbb{R}) \, du\\
            &=COT_{h}(\mu,\nu).
        \end{align*}
        

        Now, let us do the proof in general:
        \begin{align*}
            \int_0^1 h(|M_\mu^\nu(x)-x|_\mathbb{R}) \, d\mu(x)
            &=\int_0^1 h(|F_{\nu,x_{cut}}^{-1}(F_{\mu,x_{cut}}(x-x_{cut}))-(x-x_{cut})|_\mathbb{R}) \, d\mu(x)\\
            &=            
            \int_0^1 h(|F_{\nu,x_{cut}}^{-1}(y)-F_{\mu,x_{0}}^{-1}(y)|_\mathbb{R}) \, d (F_{\mu,x_{cut}}(\cdot-x_{cut}))_\#\mu(y)\\
            &=\int_{0}^{1} h(|F_{\nu,x_{cut}}^{-1}(u)-F_{\mu,x_{cut}}^{-1}(u)|_\mathbb{R}) \, du\\
            &=COT_{h}(\mu,\nu).
        \end{align*}
        In the last equality we have used that $F_{\mu,x_{cut}}(\cdot-x_{cut})_\#\mu$ is the Lebesgue measure (see \eqref{eq: pushforward F_mu,cut}).
        
        \item     Using the definition of the generalized inverse (quantile function), we have
    \begin{align*}
        M_{\mu}^{\nu}(t)&=F_{\nu,x_{cut}}^{-1}(F_{\mu,x_{cut}}(x-x_{cut}))+x_{cut}
        \\
        &=\inf\{x': \, F_{\nu,x_{cut}}(x')>F_{\mu,x_{cut}}(x-x_{cut})\}+x_{cut}\\
        &=\inf\{x': \, F_\nu(x'+x_{cut})-F_\nu(x_{cut})>F_\mu(x)-F_\mu(x_{cut})\}+x_{cut}\\
        &=\inf\{x': \, F_\nu(x'+x_{cut})>F_\mu(x)-F_\mu(x_{cut})+F(x_{cut})\}+x_{cut}\\
        &=\inf\{x': \, F_\nu(x'+x_{cut})>F_\mu(x)-\alpha_{\mu,\nu}\}+x_{cut}\\
        &=\inf\{y-x_{cut}: \, F_\nu(y)>F_\mu(x)-\alpha_{\mu,\nu}\}+x_{cut}\\
        &=\inf\{y: \, F_\nu(y)>F_\mu(x)-\alpha_{\mu,\nu}\}+x_{cut}-x_{cut}\\
        &=F_\nu^{-1}(F_\mu(x)-\alpha_{\mu,\nu}).
    \end{align*}    
        \item This part follows from the previous item as the right-hand side of \eqref{eq: mong with alpha} does not depend on any minimizer of \eqref{eq: cot solution x0}.
        \item From \cite[Theorem 13]{mccann2001polar}, there exists a unique optimal Monge map for the optimal transport problem on the unit circle.
        Therefore, $M_\mu^\nu$ is the unique optimal transport map from $\mu$ to $\nu$.     
        For the quadratic case $h(x)=|x|^2$, we refer for example to \cite[Th. 1.25, Sec. 1.3.2]{sant2015}).         
        Moreover, in this particular case, there exists a function $\varphi$ such that $M_\mu^\nu(x)=x-\nabla\varphi(x)$, where $\varphi$ is a \textit{Kantorovich potential} (that is, a solution to the dual optimal transport problem on $\mathbb{S}^1$) and the sum is modulo $\mathbb{Z}$.

        \item The identity $(M_\mu^\nu)^{-1}=(M_\nu^\mu)$ holds from the symmetry of the cost \eqref{eq: COT def} that one should optimize. Also, it can be verified using \eqref{eq: mong with alpha} and the fact that from \eqref{eq: equation for alpha} $\alpha_{\mu,\nu}=-\alpha_{\nu,\mu}$:
        \begin{align*}
        M_\nu^\mu\circ M_\mu^\nu(x)&=
            F_\mu^{-1}\left(F_\nu(F_\nu^{-1}(F_\mu(x)-\alpha_{\mu,\nu}))-\alpha_{\nu,\mu}\right)\\
            &=F_\mu^{-1}\left(F_\mu(x)-\alpha_{\mu,\nu}+\alpha_{\mu,\nu}\right)=x.
        \end{align*}
    \end{enumerate}
\end{proof}


\begin{proposition}[Properties of the LCOT-Embedding]\label{prop: properties of embedding}\;
Let $\mu\in \mathcal{P}(\mathbb{S}^1)$ be absolutely continuous with respect to the Lebesgue measure on $\mathbb{S}^1$, and let $\nu\in \mathcal{P}(\mathbb{S}^1)$.
\begin{enumerate}
    \item $\widehat{\mu}^{\mu,h}\equiv 0$. 
    \item $\widehat{\nu}^{\mu,h}(x)\in [-0.5,0.5]$ \quad for every $x\in[0,1)$.
    \item Let $\nu_1,\nu_2\in\mathcal{P}(\mathbb{S}^1)$ with $\nu_1$ that does not give mass to atoms,  then the map
    \begin{equation}
        M:=(\widehat{\nu_2}^{\mu,h}-\widehat{\nu_1}^{\mu,h})\circ ((\widehat{\nu_1}^{\mu,h}+\mathrm{id})^{-1}) +\mathrm{id},  
    \end{equation}      
    satisfies $M_\#\nu_1=\nu_2$ (however, it is not necessarily an optimal circular transport map).

    \end{enumerate}
\end{proposition}

\begin{proof}[Proof of Proposition \ref{prop: properties of embedding}]\,

\begin{enumerate}
    \item It trivially holds that the optimal Monge map from the distribution $\mu$ to itself is the identity $\mathrm{id}$, or equivalently, that the optimal displacement is zero for all the particles.  
    \item It holds from the fact of being the optimal displacement, that is,
    \begin{equation*}
        COT_{h}(\mu,\nu)= \inf_{M:\, M_{\#}\mu=\nu}\int_{\mathbb{S}^1}h(|M(x)-x|_{\mathbb{S}^1}) \, d\mu(x)=\int_{\mathbb{S}^1}h(|\widehat{\nu}^{\mu,h}(x)|_{\mathbb{S}^1}) \, d\mu(x),
    \end{equation*}
    and from the fact that $|z|_{\mathbb{S}^1}$ is at most $0.5$.
    \item We will use that $\widehat{\nu}^{\mu,h}=M_\mu^\nu-\mathrm{id}$, and that $(M_\mu^\nu)^{-1}=M_\nu^\mu$:

    \begin{align*}
        M(x)&=(\widehat{\nu_2}^{\mu,h}-\widehat{\nu_1}^{\mu,h})\circ M_{\nu_1}^{\mu}(x) +x\\
        &=(M_\mu^{\nu_2}-M_\mu^{\nu_1})\circ M_{\nu_1}^{\mu}(x) +x\\
        &=M_\mu^{\nu_2}\circ M_{\nu_1}^{\mu}(x) -x +x\\
        &=M_\mu^{\nu_2}\circ M_{\nu_1}^{\mu}(x).
    \end{align*}
Finally, notice that 
$$(M_\mu^{\nu_2}\circ M_{\nu_1}^{\mu})_\#\nu_1=(M_\mu^{\nu_2})_\#\left( (M_{\nu_1}^{\mu})_\#\nu_1\right)=(M_\mu^{\nu_2})_\#\mu=\nu_2.$$
\end{enumerate}
\end{proof}

Now, we will proceed to prove Theorem \ref{thm: distance}. By having this result, it is worth noticing that $LCOT_{\mu,p}(\cdot,\cdot)^{1/p}$ endows $\mathcal{P}(\mathbb{S}^1)$ with a metric-space structure. 
The proof is based on the fact that we have introduced an explicit embedding and then we have considered an $L^p$-distance. It will follow that we have defined a kernel distance (that is in fact positive semidefinite). 

\begin{proof}[Proof of Theorem \ref{thm: distance}]

From \eqref{eq: LCOT1}, it is straightforward to prove the 
 symmetric property and non-negativity. 

If $\nu_1=\nu_2$, by the uniqueness of the optimal COT map (see Theorem \ref{thm: cot embedding}, Part 3), we have $\widehat{\nu_1}^{\mu,h}=\widehat{\nu_2}^{\mu,h}$. 
Thus, $LCOT_{\mu,h}(\nu_1,\nu_2)=0$.

For the reverse direction, if $LCOT_{\mu,h}(\nu^1,\nu^2)=0$, then 
$$h(\min_{k\in\mathbb{Z}}\{|\widehat{\nu_1}^{\mu,h}(x)-\widehat{\nu_2}^{\mu,h}(x)+k|\})=0 \qquad \mu-\text{a.s.}$$
Thus, 
$$\widehat{\nu_1}^{\mu,h}(x)\equiv_1 \widehat{\nu_2}^{\mu,h}(x) \qquad \mu-\text{a.s.}$$
(where $\equiv_1$ stands for the equality modulo $\mathbb{Z}$).
That is, 
$$
M_{\mu}^{\nu_1}(x)=\widehat{\nu_1}^{\mu,h}(x)+x\equiv_1 \widehat{\nu_2}^{\mu,h}(x)+x=M_{\mu}^{\nu_2}(x) \qquad\mu\text{ a.s.}$$
Let $S\subseteq[0,1)$ denote the set of $x$ such that the equation above holds, we have $\mu(S)=1, \mu(\mathbb{S}^1\setminus S)=0$. Equivalently, for any (measurable) $B\subseteq\mathbb{S}^1$, $\mu(B\cap S)=\mu(B)$. Pick any Borel set $A\subseteq \mathbb{S}^1$, we have: 
\begin{align}
\nu_1(A)
&=\mu\left((M_{\mu}^{\nu_1})^{-1}(A)\right)\nonumber\\
&=\mu\left((M_{\mu}^{\nu_1})^{-1}(A) \cap S\right)\nonumber\\
&=\mu\left((M_{\mu}^{\nu_2})^{-1}(A) \cap S\right)\nonumber\\ 
&=\mu((M_{\mu}^{\nu_2})^{-1}(A))\nonumber\\
&=\nu_2(A)
\end{align}
where the first and last equation follows from the fact $M_{\mu}^{\nu_1},M_{\mu}^{\nu_2}$ are push forward mapping from $\mu$ to $\nu_1$, $\nu_2$ respectively. 

Finally, we verify the triangular inequality. Here we will use that $h(x)=|x|^p$, for $1\leq p<\infty$. Let $\nu_1,\nu_2,\nu_3\in\mathcal{P}(\mathbb{S}^1)$,
\begin{align*}
{LCOT}_{\mu,p}(\nu_1,\nu_2)^{1/p}&=\left(\int_{0}^1 (|\widehat{\nu_1}(t)-\widehat{\nu_2}(t)|_{\mathbb{S}^1})^{p} \, d\mu(t)\right)^{1/p} \nonumber \\ 
&\leq \left(\int_{0}^1 \left(|\widehat{\nu_1}(t)-\widehat{\nu_3}(t)|_{\mathbb{S}^1}+|\widehat{\nu_3}(t)-\widehat{\nu_2}(t)|_{\mathbb{S}^1}\right)^p \, d\mu(t)\right)^{1/p}\nonumber\\ 
&\leq \left(\int_0^1|\widehat{\nu_1}(t)-\widehat{\nu_3}(t)|_{\mathbb{S}^1}^p \,  d\mu(t)\right)^{1/p}+\left(\int_0^1|\widehat{\nu_3}(t)-\widehat{\nu_2}(t)|_{\mathbb{S}^1}^p \, d\mu(t) \right)^{1/p}\nonumber \\
&=LCOT_{\mu,p}(\nu_1,\nu_3)^{1/p}+LCOT_{\mu,p}(\nu_2,\nu_3)^{1/p}
\end{align*}
where the last inequality holds from Minkowski inequality. 
\end{proof}

\subsection{Understanding the relation between the minimizers of \eqref{eq: cot solution x0} and \eqref{eq: cot alpha}}\label{sec: appendix alpha vs x_cut}

We briefly revisit the discussion in Section \eqref{sec: problem set up}, specifically in Remark \eqref{remark: not unique x_cut but unique alpha}, concerning the optimizers $x_{{cut}}$ and $\alpha_{\mu,\nu}$ of \eqref{eq: cot solution x0} and \eqref{eq: cot alpha}, respectively.

Assuming minimizers exist for \eqref{eq: cot solution x0} and \eqref{eq: cot alpha}, we first explain why we adopt the terminology \textit{"cutting point"} ($x_{{cut}}$) for a minimizer of \eqref{eq: cot solution x0} and not for the minimizer $\alpha_{\mu,\nu}$ of \eqref{eq: cot alpha}. On the one hand, the cost function presented in \ref{eq: cot solution x0} is given by
\begin{equation}\label{eq: cost x_0}
  \text{Cost}(x_0):=\int_{0}^1h(|F_{\mu,x_0}^{-1}(x)-F_{\nu,x_0}^{-1}(x)|_{\mathbb{R}}) \, dx.  
\end{equation}
We seek to minimize over $x_0 \in [0,1) \sim \mathbb{S}^1$, aiming to find an optimal $x_0$ that affects both CDFs $F_\mu$ and $F_\nu$. By looking at the cost \ref{eq: cost x_0}, for each fixed $x_0$, we change the system of reference by adopting $x_0$ as the origin. Then, once an optimal $x_0$ is found (called $x_{\text{cut}}$), it leads to the optimal transportation displacement, providing a change of coordinates to unroll the CDFs of $\mu$ and $\nu$ into $\mathbb{R}$ and allowing the use the classical Optimal Transport Theory on the real line (see Section \eqref{sec: closed formula for COT} and the proofs in Appendix \eqref{sec: Appendix proofs}).
On the other hand, the cost function in \ref{eq: cot alpha} is
$$ \text{Cost}(\alpha):=\int_{0}^1h(|F_{\mu}^{-1}(x+\alpha)-F_{\nu}^{-1}(x)|_{\mathbb{R}}) \, dx,$$
and the minimization runs over every real number $\alpha$. Here, the shift by $\alpha$ affects only one of the CDFs, not both. Therefore, it will not allow for a consistent change in the system of reference. This is why we do not refer to $\alpha$ as a cutting point in this paper, but we do refer to the minimizer of \eqref{eq: cot solution x0} as $x_{{cut}}$.

Finally, Figure \ref{fig: alpha vs x_cut} below is meant to provide a visualization of Remark \ref{remark: not unique x_cut but unique alpha}, that is, to show through an example that, when minimizers for \ref{eq: cot solution x0} and  \eqref{eq: cot alpha} do exist, while one could have multiple minimizers of \ref{eq: cot solution x0}, the minimizer of \ref{eq: cot alpha} is unique.
\begin{figure}[ht]
    \centering
    \includegraphics[width=1\linewidth]{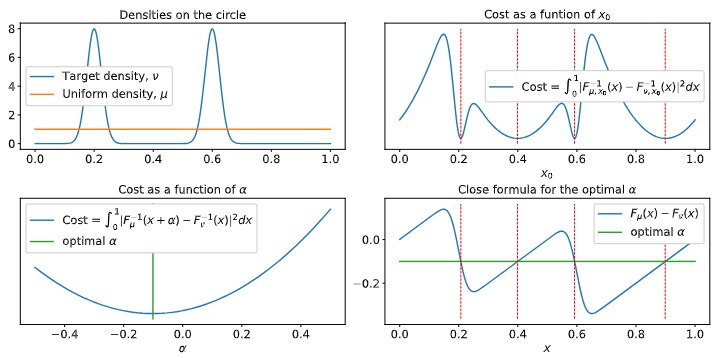}
    \vspace{-.2in}
    \caption{Top left: Uniform density, $\mu$, and a random target density $\nu$ on $\mathbb{S}^1$. Top right: The circular transportation cost  $\int_{0}^1 |F_{\mu,x_0}^{-1}(x)-F_{\nu,x_0}^{-1}(x)|^2 \, dx$ is depicted as a function of the cut,  $x_0$, showing that the optimization in \eqref{eq: cot solution x0} can have multiple minimizers. Bottom right: Following \eqref{eq: equation for alpha}, we depict the difference between the two CDFs, $F_\mu(x)-F_\nu(x)$, for each $x\in[0,1)\sim\mathbb{S}^1$. As can be seen, for the optimal cuts (dotted red lines), the difference is constant, indicating that the optimal $\alpha$ for \eqref{eq: cot alpha} is unique.  Bottom left: The optimizer for the circular transportation cost in \eqref{eq: cot alpha} is unique, and given that $\mu$ is the uniform measure, it has a closed-form solution $\mathbb{E}(\nu)-\frac{1}{2}$.}
    \label{fig: alpha vs x_cut}
\end{figure}

\subsection{Time complexity of Linear COT}\label{sec: appendix numerics} 

In this section, we assume that we are given discrete or empirical measures.\footnote{It is worth mentioning that for some applications, the LCOT framework can be also used for continuous densities,  as in the case of the CDT \cite{park2018cumulative}.}

First, we mention that according to \cite[Section 6]{delon2010fast}, given two non-decreasing step functions $F$ and $G$ represented by 
$$[\, [x_1,\dots,x_{N_1}], \;[F(x_1),\dots,F(x_{N_1})] \; ] \qquad \text{and} \qquad [ \,[y_1,\dots,y_{N_2}], \,[G(y_1),\dots,G(y_{N_2})] \, ],$$
the computation of an integral of the form
$$\int c(F^{-1}(x),G^{-1}(x)) \, dx$$
requires $\mathcal{O}(N_1+N_2)$ evaluations of a given cost function $c(\cdot,\cdot)$.

Now, by considering the reference measure $\mu=Unif(\mathbb{S}^1)$ we will detail our algorithm for computing $LCOT(\nu_1,\nu_2)$. Let us assume that  $\nu_1,\nu_2$ are two discrete probability measures on $\mathbb{S}^1$ having $N_1$ and $N_2$ masses, respectively. We represent these measures $\nu_i=\sum_{j=1}^{N_i}m_j^i\delta_{x_j^i}$ (that is, $\nu_1$ has mass $m_{j}^1$ at location $x_j^1$ for $j=1,\dots, N_1$, and analogously for $\nu_2$) as arrays of the form
$$\nu_i=[ \ [x^i_1,\dots,x^i_{N_i}], \ [m_1^i,\dots ,m_{N_i}^i] \ ], \qquad i=1,2.$$

Algorithm to compute LCOT:
\begin{enumerate}
    \item For $i=1,2$, compute $\alpha_{\mu,\nu_i}=\mathbb{E}(\nu_i)-1/2$. 
    \item For $i=1,2$, represent $F_{\nu_i}(\cdot)+\alpha_{\mu,\nu_i}$ as the arrays
    $$[\ [x^i_1,\dots,x^i_{N_i}],\ [c_1^i,\dots ,c_{N_i}^i] \ ]$$
    where
    $$c_1^i:=m_1^i+\alpha_{\mu,\nu_i}, \qquad c_j^i:=c_{j-1}^i+m_j^i, \quad \text{for } j=2,\dots, N_i.$$
    \item Use that 
    
        $$F_\nu^{-1}(x-\alpha_{\mu,\nu})=(F_\nu(\cdot)+\alpha_{\mu,\nu})^{-1}(x),$$

and the algorithm provided in \cite[Section 6]{delon2010fast} mentioned above with $F=F_{\nu_1}(\cdot)+\alpha_{\mu,\nu_1}$ and $G=F_{\nu_2}(\cdot)+\alpha_{\mu,\nu_2}$ to compute
    \begin{align*}
        LCOT(\nu_1,\nu_2)&=\|\widehat{\nu_1}-\widehat{\nu_2}\|^2_{L^2(\mathbb{S}^1)}
        \\
        &=\int_0^1 |\left(F_{\nu_1}^{-1}(x-\alpha_{\mu,\nu_1})-x\right)-\left(F_{\nu_2}^{-1}(x-\alpha_{\mu,\nu_2})-x\right)|_{\mathbb{S}^1}^2 \, dx\\
        &=\int_0^1 |(\underbrace{F_{\nu_1}(\cdot)+\alpha_{\mu,\nu_1}}_{F})^{-1}(x)-(\underbrace{F_{\nu_2}(\cdot)+\alpha_{\mu,\nu_2}}_{G})^{-1}(x)|_{\mathbb{S}^1}^2 \, dx
    \end{align*}
\end{enumerate}
Each step requires $\mathcal{O}(N_1+N_2)$ operations. Therefore, the full algorithm to compute $LCOT(\nu_1,\nu_2)$ is of order $\mathcal{O}(N_1+N_2)$. 


\subsection{Experiments}
 
The following Figure \ref{fig: mds_scale} is from an extra experiment analogous to Experiment 1 but for a different family of measures (Figure \ref{fig: mds_scale} Left). We include it to have an intuition of how the LCOT behaves under translations and dilations of an initial von Mises density.
 
\begin{figure}[ht]
    \centering
    \includegraphics[width=\linewidth]{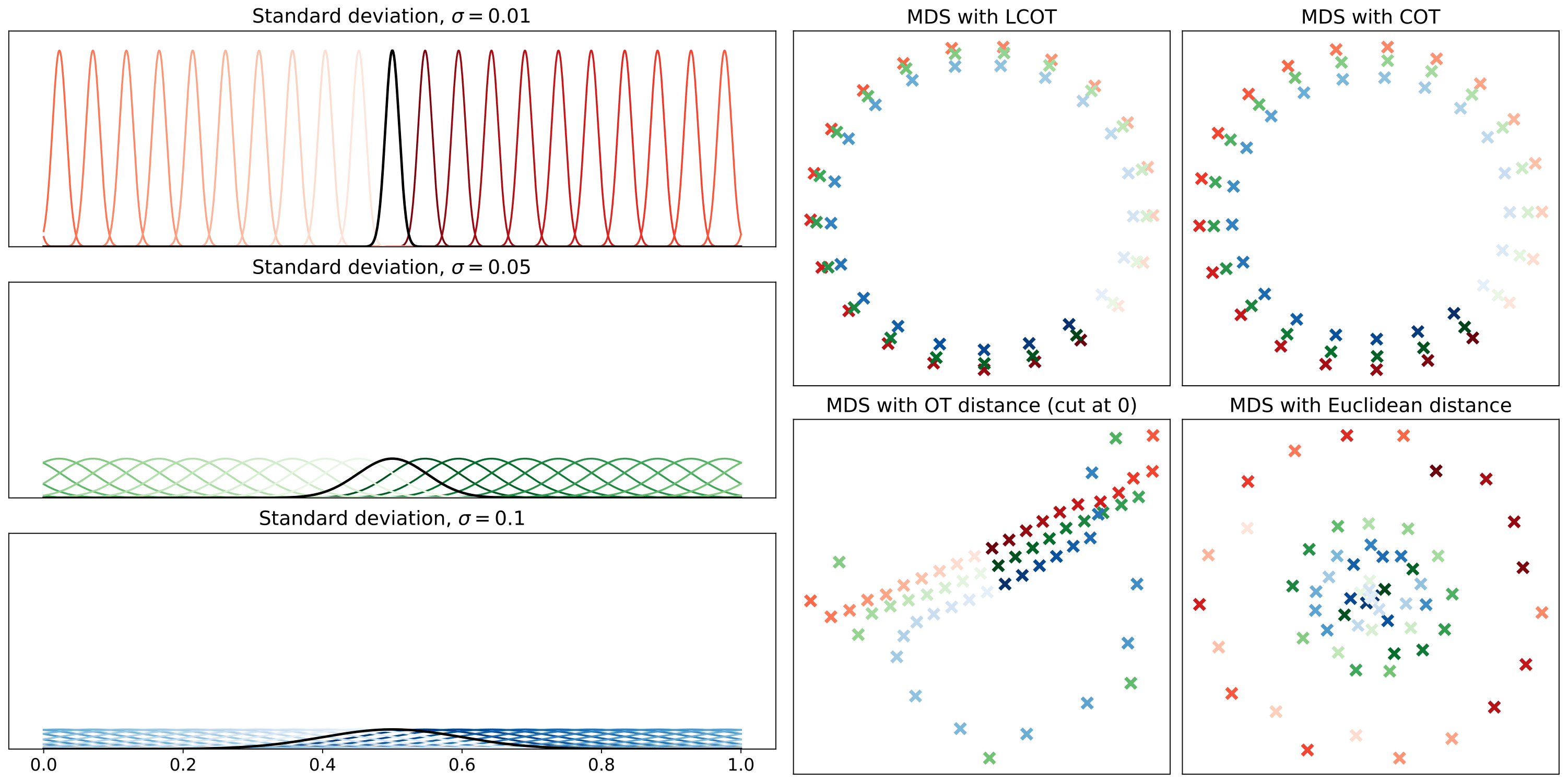}
    \caption{MDS for embedding classes of probability densities into an Euclidean space of dimension $2$ where the original pair-wise distances (COT-distance, LOT-distance, Euclidean or $L^2$-distance) are preserved as well as possible.}
    \label{fig: mds_scale}
\end{figure}

\subsection{Understanding the embedding in differential geometry} \label{sec: appendix log exp}

Our embedding $ \nu \mapsto \widehat{\nu} $ as given by equation \eqref{eq: embedding general} aligns with the definition of the Logarithm function presented in \cite[Definition 2.7]{sarrazin2023linearized}. To be specific, for $ \mu,\nu \in \mathbb{S}^1 $ and the Monge mapping $M_{\mu}^\nu $, the Logarithm function as introduced in \cite{sarrazin2023linearized} is expressed as:
\[ \mathcal{P}_2(\mathbb{S}^1)\ni \nu \mapsto \log^{COT}_{\mu}(\nu) \in L^2(\mathbb{S}^1,T\mathbb{S}^1;\mu). \]

Here, the tangent bundle of $ \mathbb{S}^1 $ is represented as $$ T\mathbb{S}^1 := \{(x, T_x(\mathbb{S}^1))| \, x\in\mathbb{S}^1\}, $$ where $ T_x(\mathbb{S}^1) $ denotes the tangent space at the point $ x\in\mathbb{S}^1 $. The space $ L^2(\mathbb{S}^1,T\mathbb{S}^1;\mu) $ is the set of vector fields on $ \mathbb{S}^1$ with squared norms (based on the metric on $ T\mathbb{S}^1 $), that are $ \mu $-integrable. The function (vector field) $ \log^{COT}_{\mu}(\nu) $ is defined as:
\[ \log^{COT}_{\mu}(\nu) := (\mathbb{S}^1\ni x \mapsto (x,v_x)) \in T\mathbb{S}^1, \]
where $ v_x \mapsto T_x(\mathbb{S}^1) $ is the initial velocity of the unique constant speed geodesic curve $ x \mapsto T_\mu^\nu(x) $.

The relation between $ \log_\mu^{COT}(\nu) $ and $ \widehat{\nu} $ in \eqref{eq: embedding} can be established as follows:
For any $ x $ in $ \mathbb{S}^1 $, the spaces $ T_x(\mathbb{S}^1)$ and $\mathbb{S}^1 $ can be parameterized by $ \mathbb{R}$ and $[0,1)$, respectively. Then, the unique constant speed curve $ x \mapsto M_{\mu}^\nu(x) $ is given by:
\[ x(t) := x + t(M_{\mu}^\nu(x) - x), \qquad \forall t\in[0,1].\]
Then, the initial velocity is $ M_{\mu}^\nu(x) - x $.
Drawing from \eqref{eq: circ opt disp}, Theorem \ref{thm: cot embedding}, and Proposition \ref{prop: properties of embedding}, we find $ \widehat{\nu}(x) = M_{\mu}^\nu(x) - x $ for all $ x $ in $ \mathbb{S}^1 $, making $ \widehat{\nu} $ and $ \log_\mu^{COT}(\nu) $ equivalent.

However, it is important to note that while $ \log_\mu^{COT} $ is defined for a generic (connected, compact, and complete\footnote{In \cite{sarrazin2023linearized}, the Riemannian manifold is not necessarily compact. However, the measures $ \mu, \nu $ must have compact support sets. For brevity, we have slightly overlooked this difference.}) manifold, it does not provide a concrete computational method for the embedding $ \log_\mu^{COT} $. Our focus in this paper is on computational efficiency, delivering a closed-form formula.

Regarding the embedding space, in \cite{sarrazin2023linearized}, the space $L^2(\mathbb{S}^1, T\mathbb{S}^1;\mu)$ is equipped with the $L^2$, induced by $T\mathbb{S}^1$. Explicitly, for any $f$ belonging to $L^2(\mathbb{S}^1, T\mathbb{S}^1; \mu)$,
\begin{align*}
  \|f\|^2 = \int_0^1 \|f(x)\|_x^2 dx = \int_0^1 |f(x)|^2 \, dx,
\end{align*}
where $\|f(x)\|_x^2$ represents the norm square in the tangent space $T_x(\mathbb{S}^1)$ of the vector $f(x)$. By parameterizing $\mathbb{S}^1$ and $T_x(\mathbb{S}^1)$ as $[0,1)$ and $\mathbb{R}$, respectively, this squared norm becomes $|f(x)|^2$. Consequently, $L^2(\mathbb{S}^1, T\mathbb{S}^1; \mu)$ becomes an inner product space, whereby the expression  (polarization identity) $\|f+g\|^2 - \|f\|^2 - \|g\|^2$ establishes an inner product between $f$ and $g$.

However, in this paper, the introduced embedding space $L^2(\mathbb{S}^1, d\mu)$ presented in \eqref{eq: embedding space}. This space uses the $L^2$-norm on the circle, defined for each $f$ in $L^2(\mathbb{S}^1, d\mu)$ as:
\[
\|f\|_{L^2(S^1; d\mu)}^2 = \int_{\mathbb{S}^1} |f(x)|^2_{\mathbb{S}^1} \,  d\mu.
\]
Unlike the previous space, this does not induce an inner product (in fact, $|\cdot|_{\mathbb{S}^1} $ is not a norm). As such, 
throughout this paper, we term our embedding as a ``linear embedding'' rather than a ``Euclidean embedding''.

\end{document}